\documentclass{article}

\usepackage{arxiv}

\usepackage[utf8]{inputenc} % allow utf-8 input
\usepackage[T1]{fontenc}    % use 8-bit T1 fonts
\usepackage{hyperref}       % hyperlinks
\usepackage{url}            % simple URL typesetting
\usepackage{booktabs}       % professional-quality tables
\usepackage{amsfonts}       % blackboard math symbols
\usepackage{nicefrac}       % compact symbols for 1/2, etc.
\usepackage{microtype}      % microtypography
\usepackage{lipsum}		% Can be removed after putting your text content
\usepackage{graphicx}
\usepackage{natbib}
\usepackage{doi}

\usepackage{subfigure}
\usepackage{amsmath}
\usepackage{amssymb}
\usepackage{mathtools}
\usepackage{amsthm}
\usepackage{soul}
\usepackage{float}

\theoremstyle{plain}
\newtheorem{theorem}{Theorem}%[section]

\newtheorem{definition}{Definition}
\newtheorem{assumption}{Assumption}
\theoremstyle{remark}

\title{Discovering Causality for Efficient Cooperation in Multi-Agent Environments}

\date{} 					% Or removing it

\author{{Rafael Pina}\thanks{Corresponding author}, \hspace{1mm} Varuna De Silva\hspace{1mm} and\hspace{1mm} Corentin Artaud \\
	Institute for Digital Technologies\\
	Loughborough University London\\
	London, United Kingdom \\
	\texttt{\{r.m.pina, v.d.de-silva, c.artaud2\}@lboro.ac.uk} \\
}

\renewcommand{\headeright}{}
\renewcommand{\undertitle}{}
\renewcommand{\shorttitle}{Discovering Causality for Efficient Cooperation in Multi-Agent Environments}

\hypersetup{
pdftitle={A template for the arxiv style},
pdfsubject={q-bio.NC, q-bio.QM},
pdfauthor={David S.~Hippocampus, Elias D.~Striatum},
pdfkeywords={First keyword, Second keyword, More},
}

\begin{document}
\maketitle

\begin{abstract}
In cooperative Multi-Agent Reinforcement Learning (MARL) agents are required to learn behaviours as a team to achieve a common goal. However, while learning a task, some agents may end up learning sub-optimal policies, not contributing to the objective of the team. Such agents are called lazy agents due to their non-cooperative behaviours that may arise from failing to understand whether they caused the rewards. As a consequence, we observe that the emergence of cooperative behaviours is not necessarily a byproduct of being able to solve a task as a team. In this paper, we investigate the applications of causality in MARL and how it can be applied in MARL to penalise these lazy agents. We observe that causality estimations can be used to improve the credit assignment to the agents and show how it can be leveraged to improve independent learning in MARL. Furthermore, we investigate how Amortized Causal Discovery can be used to automate causality detection within MARL environments. The results demonstrate that causality relations between individual observations and the team reward can be used to detect and punish lazy agents, making them develop more intelligent behaviours. This results in improvements not only in the overall performances of the team but also in their individual capabilities. In addition, results show that Amortized Causal Discovery can be used efficiently to find causal relations in MARL.
\end{abstract}

% keywords can be removed
\keywords{Multi-Agent Reinforcement Learning \and Multi-Agent Cooperation \and Causality Estimations \and Deep Learning}

\section{Introduction}\label{sec:intro}
Multi-Agent Reinforcement Learning (MARL) studies the behaviours of groups of agents that interact in a system. In cooperative MARL a team of agents is trained to learn strategies to solve tasks that consist of accomplishing a certain common goal as a team \cite{gupta_cooperative_2017}. 

Recently, one of the big areas of interest in the field of MARL is the study and development of value function factorisation methods \cite{sunehag_value-decomposition_2018,son_qtran_2019,wang_qplex_2021}. The ultimate objective of these methods is to learn a joint Q-function that can be efficiently decomposed into a set of agent-wise Q-values. These methods follow a paradigm known as centralised training decentralised execution (CTDE) \cite{olienoek_2008,KRAEMER201682}. This paradigm allows the agents to have access to the full state of the environment during training, but restricts them to their local observations during execution. This creates a reasonable balance between two inherent problems in MARL: the exponential growth of the action space and the non-stationarity \cite{de_witt_multi-agent_2020}. However, although this strategy has enabled to mitigate important problems in MARL, there are still some challenges that may arise \cite{neary_reward_2021,georgios_non_stat_2019}. One problem known as the lazy agent pathology can happen when the credit assignment to the agents is not correct and some of the agents of the team become lazy \cite{sunehag_value-decomposition_2018}. As the name suggests, lazy agents are agents that do not cooperate towards the common goal of the team due to learning sub-optimal policies. For instance, in the considered scenarios where the rewards received correspond to the performance of the team, every agent receives the same reward, regardless of their individual performances. An agent that becomes lazy will not contribute to the goal of the team but still receives the reward for the performance of the others. Fig. \ref{fig:causal_illustration} depicts an example where only a part of the team is responsible for the team reward but all the agents get credit.

In this paper we intend to detect and punish lazy agents in MARL by taking advantage of causality estimations. While causal relations have been widely explored in fields such as econometrics or in the analysis of time series data, studies argue that causality can be valuable to other areas in machine learning \cite{peters_2017_elements}. We intend to use causal detection to tackle the credit assignment problem in MARL and eliminate lazy agents. In this sense, we observe that causality detection can be leveraged to improve independent learning in MARL. Thus, in this work we opt to consider each agent as an independent unit where each one of them is controlled by a distinct neural network (do not use parameter sharing \cite{gupta_cooperative_2017}).

With this work, we aim to bridge the concepts of temporal causality and MARL. We demonstrate that if we can detect causal relations between individual observations and rewards in MARL we can address the problem of lazy agents. To this end, we introduce Independent Causal Learning (ICL) and show that using causality detection can be beneficial for MARL in a set of diverse environments with distinct objectives. Furthermore, we demonstrate how individual agents develop more cooperative and more intelligent behaviours when compared to simple independent learners. At the same time, we show that state-of-the-art causality detection techniques can be applied to MARL, enhancing the relevance of studying causality relations in MARL. We focus specifically on a causality detection method called Amortized Causal Discovery (ACD) \cite{lowe_amortized_2022} to introduce ACD-MARL and show that causal relations in multi-agent systems between individual observations and the team reward can be detected by ACD-MARL and used to train more specialized agents in MARL.

\section{Related Work}\label{sec:rel_work}
One of the first methods for learning Reinforcement Learning policies in multi-agent settings was introduced by \cite{tan_multi-agent_1993} as independent learners using Q-learning \cite{watkins_technical_1992}. Recently, multiple works have focused instead on value function factorisation methods \cite{rashid_qmix_2018,sunehag_value-decomposition_2018,son_qtran_2019} that aim to mitigate the credit assignment problem by learning an efficient factorisation of a joint Q-function into individual values for each agent, capable of incorporating the relative contributions of individual agents for the team objective \cite{liir_du_2019}. Other works such as \cite{foerster_counterfactual_2018} tackle the credit assignment problem using an advantage function computed by the critic of an actor-critic method. An important common aspect among these methods is that they all operate under the CTDE convention. Although such approaches have proved to be useful, centralisation implied by CTDE may not be possible in certain scenarios \cite{canese_multi-agent_2021} and there can still arise lazy agents during learning \cite{sunehag_value-decomposition_2018}. From a slightly different perspective, reward shaping mechanisms are also popular to adjust the credit given to the agents \cite{aotani_bottom-up_2021,marek_2019,Wang_2020}. Generally, these methods learn to factorise or manipulate the rewards in a way that is more suitable for the agents to learn the tasks with a more precise credit assignment.

Within cooperative MARL, it is also relevant to mention the importance of communication-based methods. Works such as \cite{sukhbaatar_learning_2016,foerster_learning_2016,de_witt_multi-agent_2020} show how communication in MARL can be used to improve the performances of the agents. When compared to CTDE, communication methods can be seen as more realistic methods since they rely more on communication between entities rather than on a centralised oracle common to all agents. Furthermore, studies have shown that CTDE can be detrimental in some cases \cite{scaling-marl-filippos-2021}.
\begin{figure}[!t]
    \centering
     \subfigure[$t=k$]{\label{fig:lj_caus_1}\includegraphics[width=0.20\textwidth]{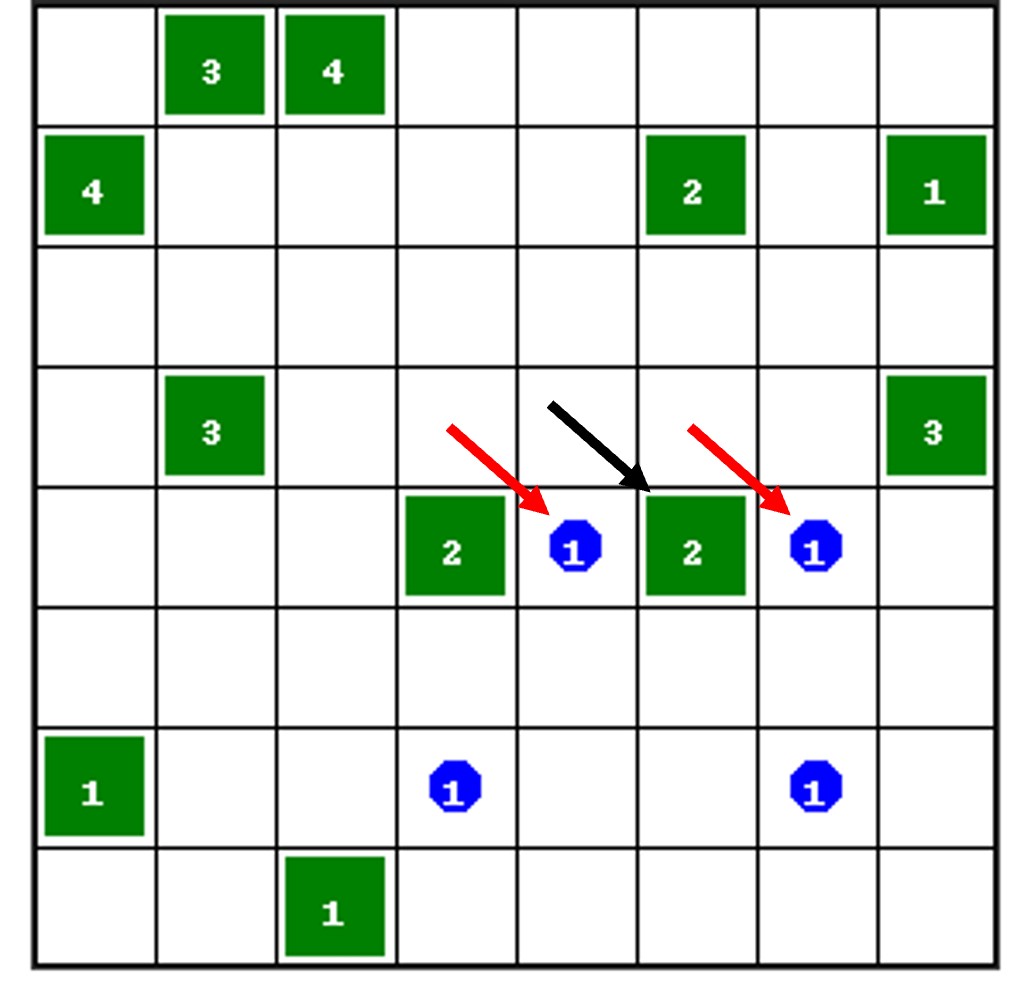}}
    \hspace{3.00mm}
    \subfigure[$t=k+1$]{\label{fig:lj_caus_2}\includegraphics[width=0.20\textwidth]{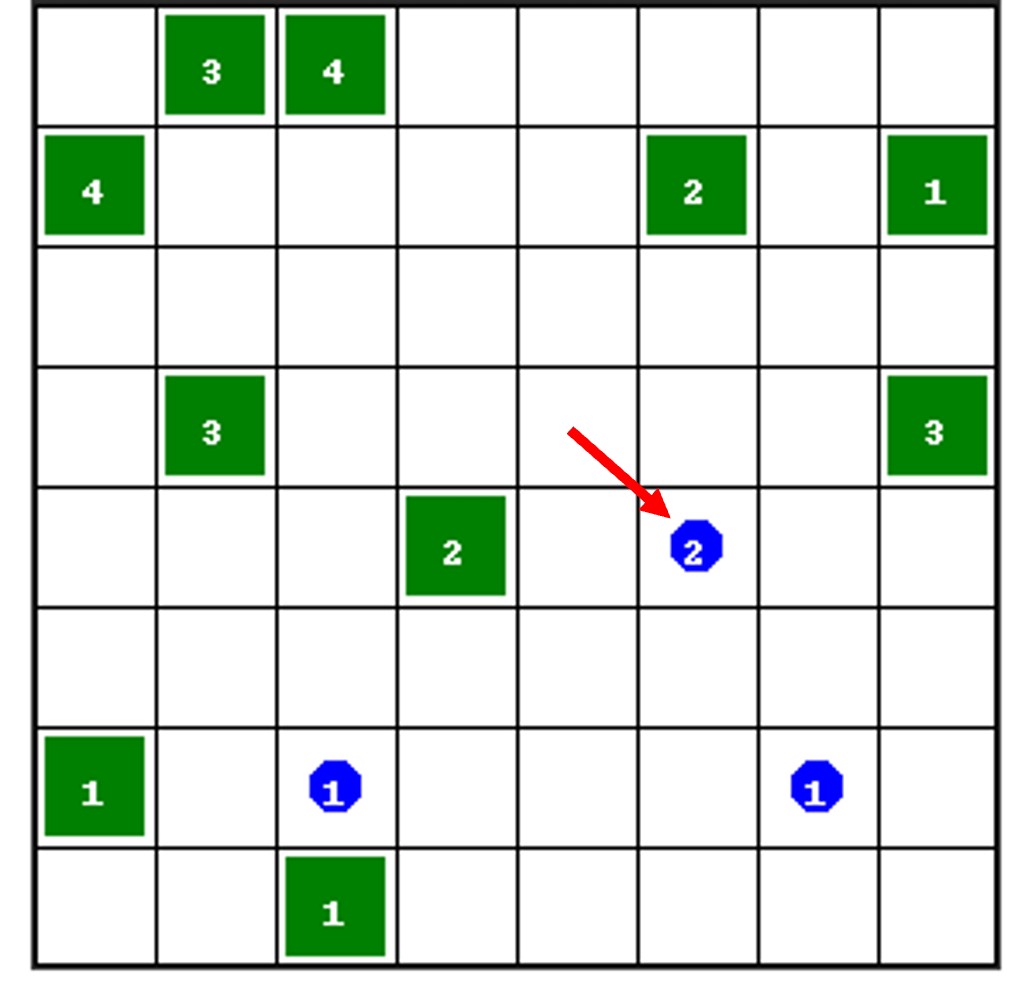}}
    \hspace{3.00mm}
    \subfigure[$t=k+2$]{\label{fig:lj_caus_3}\includegraphics[width=0.20\textwidth]{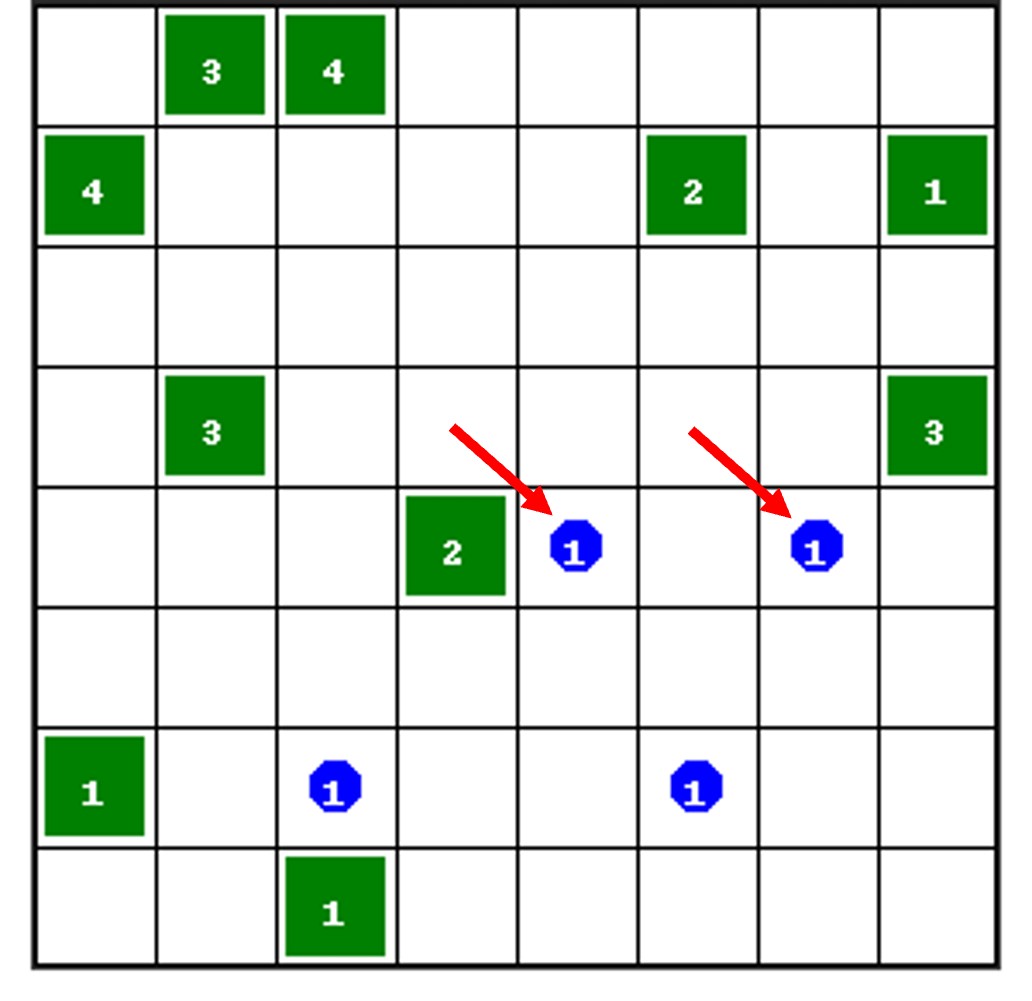}}
    \caption{An example where only two agents (red arrows) out of four are responsible for the team reward after removing the green square (black arrow). However, in MARL all the agents will be rewarded for the performance of the team. The number inside the blue circle represents the number of agents in the corresponding cell, at a timestep $t$.}
    \label{fig:causal_illustration}
\end{figure}

Comparing to the literature, we aim to tackle the credit assignment problem from a different angle, using causality detection for that end. The concept of causality has been around for some time. It was first introduced by \cite{granger_1969} as a metric to detect whether knowing the values of a given time series would help to predict a second time series. While this term has been widely used mostly in fields such as statistics and econometrics \cite{Guo2010,Seth3293,Khanna2020Economy}, causality has been growing as an area of interest in several fields. For instance, causality estimations have been applied in fields such as physics and human motion \cite{kipf_neural_2018}, or to neurology to understand how different parts of the brain interact with each other \cite{glymour_review_caus_2019}. Multiple machine learning works have also proposed different approaches based on encoder-decoder architectures to detect causality relations in time series data \cite{Zhu2020Causal,huang_causal_rl}.

With this work we aim to bridge the concepts of causality and MARL and enhance how important it can be to understand the causal relations in the underlying dynamics of multi-agent systems. Additionally, we observe that causality detection can be used to improve the cooperative behaviours of independent learners.

\section{Preliminaries}\label{sec:backg}
\subsection{Decentralised Partially Observable Markov Decision Processes (Dec-POMDPs)}\label{sec:decpomdp}
We consider fully cooperative multi-agent tasks modelled as Dec-POMDPs \cite{oliehoek_a_concise_2016}. A Dec-POMDP can be represented as a tuple $G=\langle S,A,O,Z,P,r,\gamma,N\rangle$ where the global state of the environment is denoted by $s\in S$. At each timestep, each agent $i:i\in\mathcal{N}\equiv\{1,\ldots,N\}$ chooses an action $a_i\in A$, where $A$ is the action space, forming a joint action $a=\{a_1,\ldots,a_N\}$. From this process it results a reward $r(s,a):S\times A\rightarrow\mathbb{R}$ shared by the team and followed by a transition on the state of the environment according to the probability $P(s'|s,a):S\times A\times S\rightarrow [0,1]$, where $s'$ represents the next state and $P$ models the dynamics of the environment. In a partially observable setting, at each state $s$ each agent $i$ has an individual observation $o_i\in O(s,i):S\times \mathcal{N}\rightarrow Z$ and keeps an action-observation history $\tau_i\in \mathcal{T}=\{\tau_1,\ldots,\tau_N\}$, where $\mathcal{T}$ denotes the set of action-observation histories. Let $\gamma:\gamma \in[0,1]$ represent the discount factor. The goal is to maximise the discounted return at a given timestep $t$, defined as $R_i^t=\sum_{k=0}^\infty\gamma_kr_{t+k}$ for each agent $i$, considering $k$ timesteps ahead of timestep $t$. Although the $S$ component makes part of the Dec-POMDP, in the approach proposed in this paper the agents use only their individual observations $o_i$ and the reward $r$, and do not have access to the full state $s$ of the environment at any moment. In other words, the policy $\pi_i$ of each agent only conditions on their local observations $o_i$.

\subsection{Independent Deep Q-Learning (IDQL)}\label{sec:idql}
The most straight forward method used for MARL is Independent Q-learning (IQL) \cite{tan_multi-agent_1993}. This approach applies the basics of Q-learning \cite{watkins_technical_1992} for independent learners that update their Q-functions individually with a learning rate $\alpha$ following the rule
\begin{equation}\label{eq:q_up}
    Q(s,a)=(1-\alpha)Q(s,a)+\alpha\left[r+\gamma\mathop{\mathrm{max}}_{a'}Q(s',a')\right]
\end{equation}
Motivated by the principles of Q-learning, \cite{mnih_human-level_2015} introduce Deep Q-Networks (DQNs). The key concept of this method is to combine deep neural networks with Q-learning, leading the agents to approximate their Q-functions as deep neural networks instead of simple lookup tables. This approach introduces the use of a replay buffer that stores past experiences, and a target network that stabilizes learning. The DQN is updated in order to minimize the loss \cite{mnih_human-level_2015}
\begin{equation}\label{eq:dqn_loss}
    \mathcal{L}(\theta)=\mathbb{E}_{b\sim B}\left[\big(r+\gamma\mathop{\mathrm{max}}_{a'}Q(s',a';\theta^-)-Q(s,a;\theta)\big)^2\right]
\end{equation}
where $\theta$ and $\theta^-$ are the parameters of the Q-network and a target Q-network, respectively, for a certain experience sample $b$ sampled from a replay buffer $B$. \cite{tampuu_multiagent_2015} put together the concepts of IQL and DQN to create independent learners that use DQNs. Throughout this paper, we refer to this combination as IDQL. We use it to benchmark the proposed approach and do not use the parameter sharing convention, allowing each agent to be controlled by an independent network.

\subsection{Granger Causality in Time Series}\label{sec:te}
\cite{granger_1969} introduced for the first time the concept of causality. In simple terms, saying that a given time series $X$ Granger-causes a second time series $Y$ means that by knowing the previous values of $X$ it is possible to get a better prediction of $Y$ than by using only the previous values of $Y$. Causality and correlation are two concepts that are often misunderstood. Given two time series, while they might be correlated, one may not necessarily have an impact on the other \cite{rohrer_thinking_2018}. Prompted by the advances on Granger Causality, we use these notions to formalise the link between causality and MARL scenarios in the sections ahead. Based on \cite{Tank_2021,lowe_amortized_2022}, we start by stating the formal definition of Granger Causality for non-linear systems: 
\begin{definition}\label{def:te}
Given a set of $n$ time series $X=\{x_1,\dots,x_n\}$ and a non-linear function $g_i$ that maps a set of past values to the series $i$ such that, for a timestep $t$, 
\begin{equation}
 x_i^{t+1} = g_i(x_1^{\leq t},\dots,x_n^{\leq t}) + \epsilon_i^{t+1}
\end{equation}
we say that a series $j$ Granger-causes the series $i$ if $g_i$ depends on the past values of $x_j$. Formally, this means that $\exists x_j^{'\leq t} \neq x_j^{\leq t} : g_i(x_1^{\leq t},\dots, x_j^{'\leq t},\dots, x_n^{\leq t}) \neq g_i(x_1^{\leq t},\dots, x_j^{\leq t},\dots, x_n^{\leq t})$.
\end{definition}

\subsection{Amortized Causal Discovery (ACD)}\label{sec:acd}
In the field of causal discovery, ACD \cite{lowe_amortized_2022} is a powerful method that can infer causal relations across different samples of time series with shared dynamics but different underlying causal graphs. ACD takes a key assumption that states that there exists some function $g$ that describes the dynamics of a set of samples $x$ given their past observations and a certain causal graph $G$ that models the causal relations of the system (with some noise $\epsilon$), as per $x^{t+1}=g(x^{\leq t}, G)+\epsilon^{t+1}$.

To learn the causal relations that describe the dynamics of the given samples, ACD uses an encoder-decoder architecture. The encoder is based on a Graph Neural Network \cite{kipf2017semisupervised} and estimates the distribution that models the causal relations across the input time series. This means that the latent space of the model represents the edges of the estimated causal graph for a set of time series. The decoder predicts the next timesteps of the input series, given the previous values of the series up to a step $t$, and the causal dynamics modelled by the latent space. As a result, the loss function used for the learning problem is composed by a negative log-likelihood term for the reconstruction and a KL-Divergence term to a pre-defined prior $p$ (uniform categorical) to act as a regularizer, resulting in the variational lower bound $\mathcal{L}$:
\begin{equation}\label{eq:elbo_loss}
    \mathcal{L}=\mathbb{E}_{q_{\phi(z|x)}}[\log p_\theta(x|z)]-\text{KL}[q_\phi(z|x)||p(z)]
\end{equation}

An important remark is that ACD considers directed graphs instead of undirected, meaning that the relations being learned are causal relations instead of simple correlations between time series. For instance, a causal relation from a time series $X$ to a time series $Y$ has a different meaning than a causal relation from $Y$ to $X$. This means that the causal graphs will always be modelled as a non-symmetric matrix. ACD also considers the existence of no connections between pairs of nodes, blocking these from exchanging information in the graph. Hence, the predictions of the decoder for each time series will only be based on other series that are causal related to itself. For this reason, ACD is aligned with previous work in Granger Causality \cite{lowe_amortized_2022}.

\section{Methods}\label{sec:meth}
\subsection{Independent Causal Learning (ICL)}\label{sec:icmarl}
In this section we formally propose Independent Causal Learning (ICL), a causality detection-based method for MARL. The goal of this method is show that causality detection can be used to improve independent learning in cooperative tasks. At the same time, we create a ground truth for the next method ACD-MARL. The key for ICL is the use of an agent-wise causality factor $c_i$ that each agent $i:i\in\{1,\dots,N\}$ uses to adjust credit of the team reward to itself. Each agent should be able to understand whether it is helping the team to achieve the intended team goal of the task or not. This method also encourages the agents to learn only by themselves. This can be beneficial, since in many real scenarios it is often unfeasible to provide the agents with the full information of the environment \cite{canese_multi-agent_2021}. Hence, to learn optimal policies in such scenarios, the agents may be forced to rely only on their individual observations to understand whether they are performing well or not. To motivate the proposed method, we investigate the concept of temporal causality, commonly used in the context of time series, as stated in Definition \ref{def:te}. Building up from this definition, we can then create the bridge between causality and a MARL problem.
\begin{definition}\label{def:te_2}
Let $E$ represent a certain episode sampled from a replay buffer of experiences in MARL, denoting a set of time series of $N$ observations and rewards $E=\{o_1,\dots,o_N,r\}$. From Definition \ref{def:te}, given a non-linear function $g_r$ that maps a set of past values to the series $r$, for the set $E$ we can say that a series $o_i$ Granger-causes the series $r$ if $g_r$ depends on the past values of $o_i$. Formally, we can write $\exists o_i^{'\leq t} \neq o_i^{\leq t} : g_r(o_1^{\leq t},\dots, o_i^{'\leq t},\dots, o_N^{\leq t},r) \neq g_r(o_1^{\leq t},\dots, o_i^{\leq t},\dots, o_N^{\leq t},r)$.
\end{definition}
Definition \ref{def:te_2} bridges the concepts of causality and MARL and states the motivation that supports the existence of a causality relationship between observations and rewards when we see a MARL episode as a set of time series. In the experiments section we provide evidence for the existence of causal relations in MARL by showing how ACD can be used to detect these relations. Based on Definition \ref{def:te_2}, and under the following Assumption \ref{ass:assump_1}, we can now define Theorem \ref{theo:theo1} that shows the rule for the calculation of each one of the individual causality factors and how they are used in the learning problem, not affecting convergence. We first introduce the necessary assumption for the validity of Theorem \ref{theo:theo1} that summarizes the conditions for a valid estimation of the factor $c_i$:
\begin{assumption}\label{ass:assump_1}
\sloppy \textit{Given a set $E=\{o_1,\dots,o_N,r\}$ and for each agent-wise causality factor $c_i$, $c_i(o_i,r) = 1 \implies \exists o_i^{'\leq t}\neq o_i^{\leq t} : g_r(o_1^{\leq t},\dots, o_i^{'\leq t},\dots, o_N^{\leq t},r) \neq g_r(o_1^{\leq t},\dots, o_i^{\leq t},\dots, o_N^{\leq t},r)$ and $c_i(o_i,r) = 0 \implies \forall o_i^{'\leq t}\neq o_i^{\leq t} : g_r(o_1^{\leq t},\dots, o_i^{'\leq t},\dots, o_N^{\leq t},r)=g_r(o_1^{\leq t},\dots, o_i^{\leq t},\dots, o_N^{\leq t},r)$ (with terms as Definition \ref{def:te_2}).}
\end{assumption}

\begin{theorem}\label{theo:theo1}
For a certain MARL task with $N$ agents, if at a given timestep $t$ each agent $i:i \in \{1,\ldots,N\}$ calculates an individual binary causality factor $c_i$, under Assumption \ref{ass:assump_1}, where 
\begin{equation}\label{eq:theo_eq1}
c_i(o_i,r)=\left\{
\begin{array}{ll}
    1 & o_i\ causes\ r\\
    0 & \lnot\ o_i\ causes\ r
\end{array}, i \in \{1,\ldots,N\}
\right.    
\end{equation}
and $o_i$ and $r$ denote the individual observations and the team reward at that timestep for an episode $E$, respectively, and each individual $Q_i$ is updated following the rule
\begin{equation}\label{eq:theo_eq2}
Q_i(\tau_i,a_i)=(1-\alpha)Q_i(\tau_i,a_i)+\alpha\left[c_i(\tau_i,r)\times r+\gamma\mathop{\mathrm{max}}_{a_i'}Q_i(\tau_i',a_i')\right]
\end{equation}
where $\alpha$ is the learning rate, then the convergence of the learning Q-function is preserved.
\end{theorem}
The proof can be found in appendix \ref{sec:app_proof}. In the configuration used for our approach, each agent has an individual policy network whose parameters are not shared with any other agents. In other words, in order to let agents learn independently we opt to maintain an independent network for each agent and do not use the parameter sharing convention of the policy networks \cite{gupta_cooperative_2017}, adopted by most of the related work in the CTDE paradigm \cite{rashid_qmix_2018,sunehag_value-decomposition_2018,wang_qplex_2021}. Thus, in the proposed methods and in the baseline IDQL, each agent has an independent policy network $\pi_i:i\in\{1,\ldots,N\}$ that is updated independently using the DQN loss function as in Eq. (\ref{eq:dqn_loss}), but with respect to $Q_i$. In ICL each agent adjusts the reward in the loss for the network update, as described in Eq. (\ref{eq:theo_eq2}).

\subsection{Causality Effect in the Environments Used}\label{sec:caus_envs}
In this subsection we describe the environments used for the experiments and show how $c_i$ can be intuitively estimated for ICL, when there is some prior knowledge about the task. Importantly, we use the causal detections estimated by this method as the ground truth for the ACD method that will be introduced ahead. Note that, in the tasks described, the team shares a reward that results from the sum of $N$ individual sub-rewards that are calculated in the environment and the agents do not see.
\begin{figure}
    \centering
    \subfigure[Predator-Prey]{\label{fig:env_a}\includegraphics[width=0.20\columnwidth]{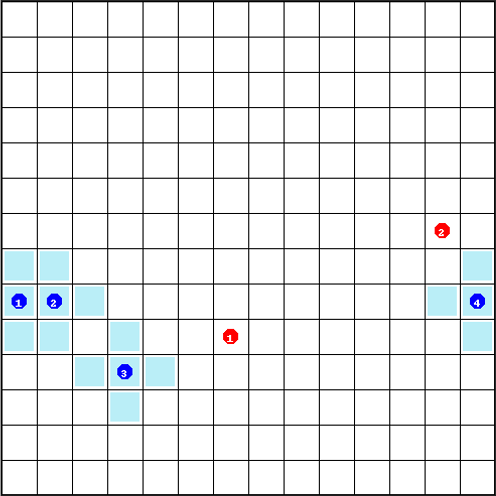}}
    \hfill
    \subfigure[Lumberjacks]{\label{fig:env_b}\includegraphics[width=0.20\columnwidth]{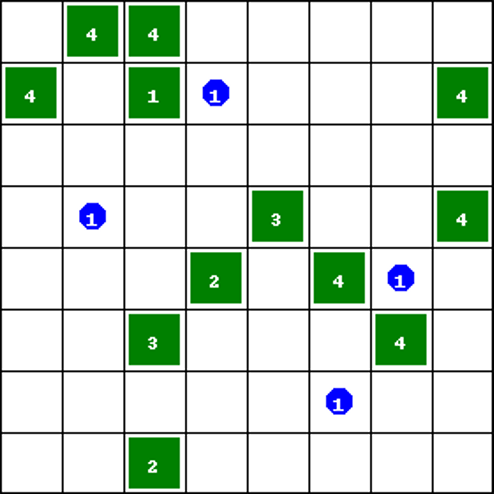}}
    \hfill
    \subfigure[SMAC-3m]{\label{fig:env_c}\includegraphics[width=0.20\columnwidth]{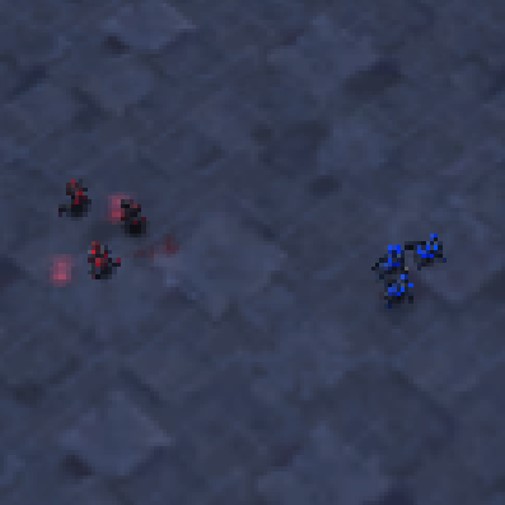}}
    \hfill
    \subfigure[SMAC-5m]{\label{fig:env_d}\includegraphics[width=0.20\columnwidth]{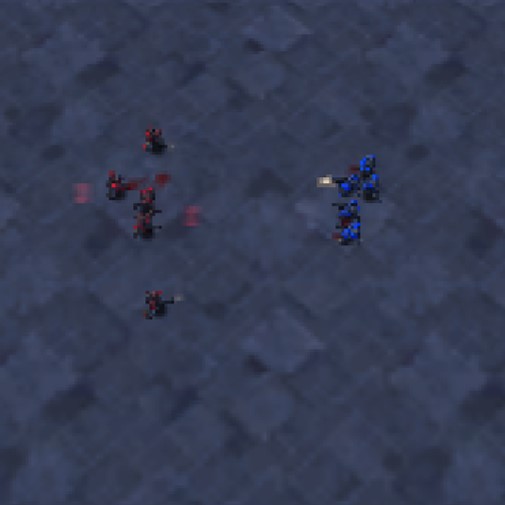}}
    \caption{Cooperative environments used for the experiments.}
    \label{fig:maps}
\end{figure}

\textbf{Predator-Prey} is a $14\times14$ grid world where a team of 4 agents needs to capture two moving preys (Fig. \ref{fig:maps}a). There is a step penalty of -0.01 and every time a prey is caught, each one of the 4 agents receives a reward of +5. However, a prey can only be caught when at least two agents do it at the same time. If only one of the agents catches the prey, then nothing happens, and the environment continues. Hence, this requires for cooperation among the elements of the team. Since only two agents are needed at the same time to capture a prey, some of the agents might not have been involved in the capture. This can lead the agents that did not participate in the capture to show later lazy behaviours. To adjust the credit assignment to the agents, we use a causality concept to relate the individual observations $o_i$ with the team reward. In this case, at each timestep $t$, each agent will only be rewarded if the following condition $C_1$ verifies: there is a positive reward (capture), and there is at least one prey in the observation mask of the agent in the moment before the capture, i.e., a boolean $C_1=(prey\ in\ o_i^{t-1}\wedge r^t > 0)$.

\textbf{Lumberjacks} is a multi-agent cooperative environment where the goal is to chop down all the trees in a $8\times8$ grid world (Fig. \ref{fig:maps}b). There are 4 agents in this environment and each tree has a randomly assigned level $l:l\in\{1,\ldots,N\}$, where $l$ agents need to step into the cell where the tree is at the same time to cut it. There is a reward of +5 for each agent every time a tree is cutdown and a step penalty of -0.1. In this environment, we only give credit to each agent when the two following conditions are satisfied: 1) there is a positive reward (tree cutdown), and there is a tree in the observation mask of the agent in the moment before the tree was cut, i.e., the boolean $C_1=(tree\ in\ o_i^{t-1}\wedge r^t > 0)$, and 2) there is a positive reward, and the number of agents seen by the agent $i$ in his observation mask (including himself) is greater or equal than the level $l$ of at least one of the existing trees observed in the observation mask of the agent $i$, $C_2=(|agents \ in\ o_i^{t-1}| \geq l\wedge r^t > 0)$.

\textbf{StarCraft Multi-Agent Challenge (SMAC)} 3m and 5m (Fig. \ref{fig:maps}c and d) are part of the SMAC collection of environments \cite{samvelyan19smac}. In these environments a team of (3 and 5) agents has to defeat the enemy team (also with 3 and 5 units) by shooting. Each agent receives an intermediate reward when an agent shoots one of the enemies, based on the damage dealt. Furthermore, there is a final reward when they win the battle. In these environments, each agent has a sight range that allows to see allies and enemies that lie within this range. The causality factor for these environments is calculated with basis on the condition: there is a positive intermediate reward (an enemy was shot), and the agent $i$ can see at least one enemy $p$ within a distance $d_p$ from his sight range $K_i$ at the moment before the shot, i.e., the boolean $C_1=(\exists\ d_p:d_p\in \left[0,K_i^{t-1}\right]\wedge r^t > 0)$. Accordingly, when there is an intermediate reward in this environment, only the agents that satisfy $C_1$ will be rewarded. This will encourage them to help the team and participate in the overall common goal. In the case of the final reward when the agents win the battle, we let all the agents receive the reward since dead agents will not be able to see other units anymore but might have contributed to win the game.
\begin{figure*}[!t]
    \centering
    \includegraphics[width=\textwidth]{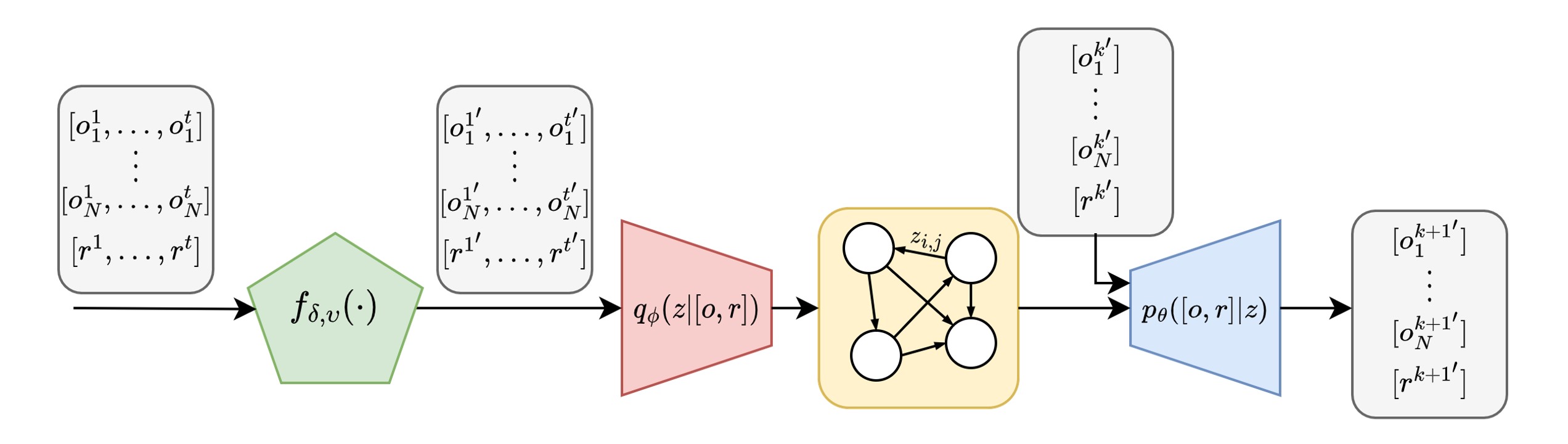}
    \caption{Proposed architecture of ACD applied to MARL problems. The encoder $q_\phi$ receives the inputs after going through function $f_{\delta,\upsilon}$ and predicts the causal graph that describes the relations $z$ among input observations and rewards (time series) over the episodes. A node $i$ is causally related to $j$ if there is a directed edge $z_{i,j}=1$ between them. Finally, a decoder $p_\theta$ predicts the next timesteps of the series given the current timestep and the predicted causal graph.}
    \label{fig:acd_arch}
\end{figure*}

\subsection{Amortized Causal Discovery (ACD) for MARL}\label{sec:acd_meth}
We now introduce a second approach that aims to support that causal relations exist in the dynamics of MARL problems and can be inferred. In this sense, we intend to show that ACD \cite{lowe_amortized_2022} can be combined with MARL problems and used to infer observation-reward causal relations within teams of agents. We refer to this method as ACD-MARL. However, it is important to note that, in order to use state-of-the-art causality detection methods in MARL such as ACD, we must abdicate of some independence in learning. More precisely, opposing to the previously described ICL, ACD requires to see all the observations and the team reward in order to infer the causal relations accurately. When given only local observations, the predictions are inaccurate due to lack of information.

For this method, we use a modified ACD architecture to be used for MARL problems. We aim to use ACD to predict the causal relations between individual observations and the team reward in a team of agents in MARL. The first stage of the method consists of gathering data from MARL to be used as input to train ACD. Hence, we collected observations and rewards of multiple successful MARL episodes, i.e., episodes where the team solves the task and receives rewards for that. After collecting successful episodes, we end up with a set of multi-variate time series, where each sub-sample contains $N$ time series of observations and one time series of rewards (the team reward). Also note that, as in the lines of the time series applications discussed in \cite{lowe_amortized_2022}, the model aims to predict one causal graph per episode that describes the causal relations between observations and the team reward for each episode.

The architecture of the proposed ACD-MARL is depicted in Fig. \ref{fig:acd_arch}. Initially, we preprocess the series and apply the function $f_{\delta,\upsilon}$ to the sampled time series and give them as input to the encoder of the model. This function describes a Savitzky-Golay Filter \cite{savitzky_smoothing_1964} with window size $\delta=t/2-(1-t/2\bmod 2)$ where $t$ is the number of timesteps in the input series, and applies a polynomial of order $\upsilon=10$ to the series. Additionally, we normalise the reward series. As mentioned, the goal of the encoder is to predict the causal graph that best describes the causal relations among the given time series. The resulting graph is a non-symmetric $(N+1)\times (N+1)$ adjacency matrix that represents the edges of the graph across the $N+1$ time series. In this work, we consider 2 edge types: edge or no edge (represented in the adjacency matrix by 1 and 0, respectively). Importantly, as it was described in the previous sections, this work focuses specifically on finding causal relations from the observations to the rewards, i.e., $o\rightarrow r$ (as per Definition \ref{def:te_2}, we aim to understand how the individual observations of each agent help to predict the future values of the rewards). As such, we stack the inputs in the shape $[o_1,o_2,\dots,o_N,r]$ and focus particularly on the last column of the learned adjacency matrix by the encoder $q_\phi(z|[o,r])$, where each value $i$ corresponds to $c_i(o_i,r)$ in the ICL ground truth described in section \ref{sec:icmarl}. The learned causal relations $z$ by the encoder are then given to the decoder $p_\theta([o,r]|z)$ together with the current values of the inputs at a timestep $k$ (Fig. \ref{fig:acd_arch}), predicting the next set of values given the current inputs and the learned relations $z$. Since we are using as input time series that are MARL episodes and, hence, have strong temporal dependencies, we use an RNN-based decoder instead of a standard MLP decoder. This model minimizes the loss described in Eq. (\ref{eq:elbo_loss}).

After the model is trained, we can use the trained encoder to predict the $o\rightarrow r$ causal relations ($c_i(o_i,r)$ for an agent $i$) within the training of MARL, similarly to ICL and the updates according to Theorem \ref{theo:theo1}, using the loss in Eq. (\ref{eq:dqn_loss}) with respect to $Q_i$.
\begin{figure*}[!t]
    \centering
    \includegraphics[width=0.4\columnwidth]{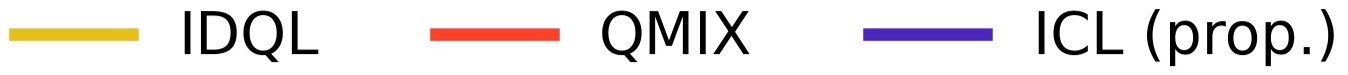}
    \\
    \vspace{0.00mm} 
    \subfigure[Predator-Prey]{\label{fig:res_a}\includegraphics[width=0.23\textwidth]{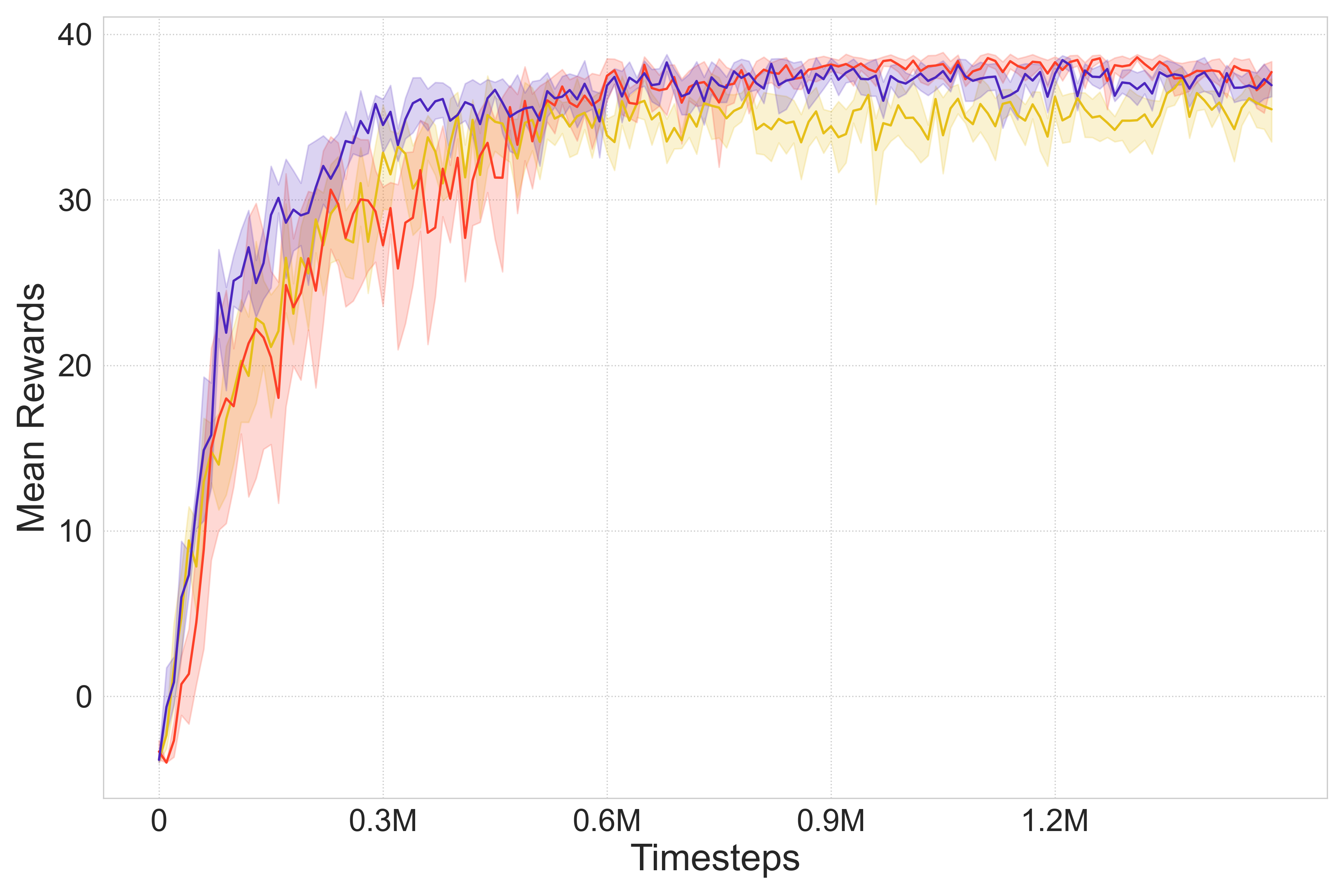}}
    \subfigure[Lumberjacks]{\label{fig:res_b}\includegraphics[width=0.23\textwidth]{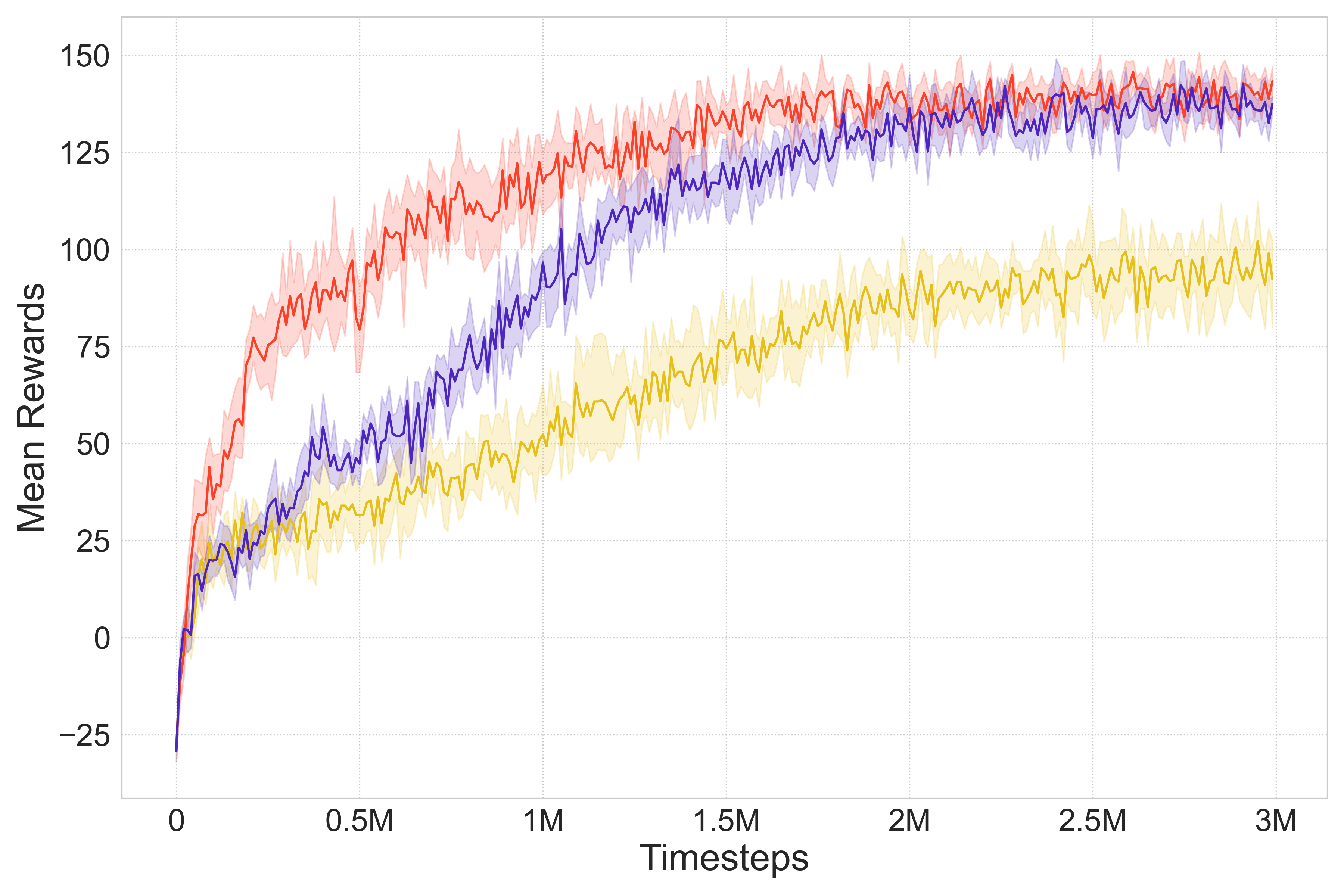}}
    \subfigure[SMAC-3m]{\label{fig:res_c}\includegraphics[width=0.23\textwidth]{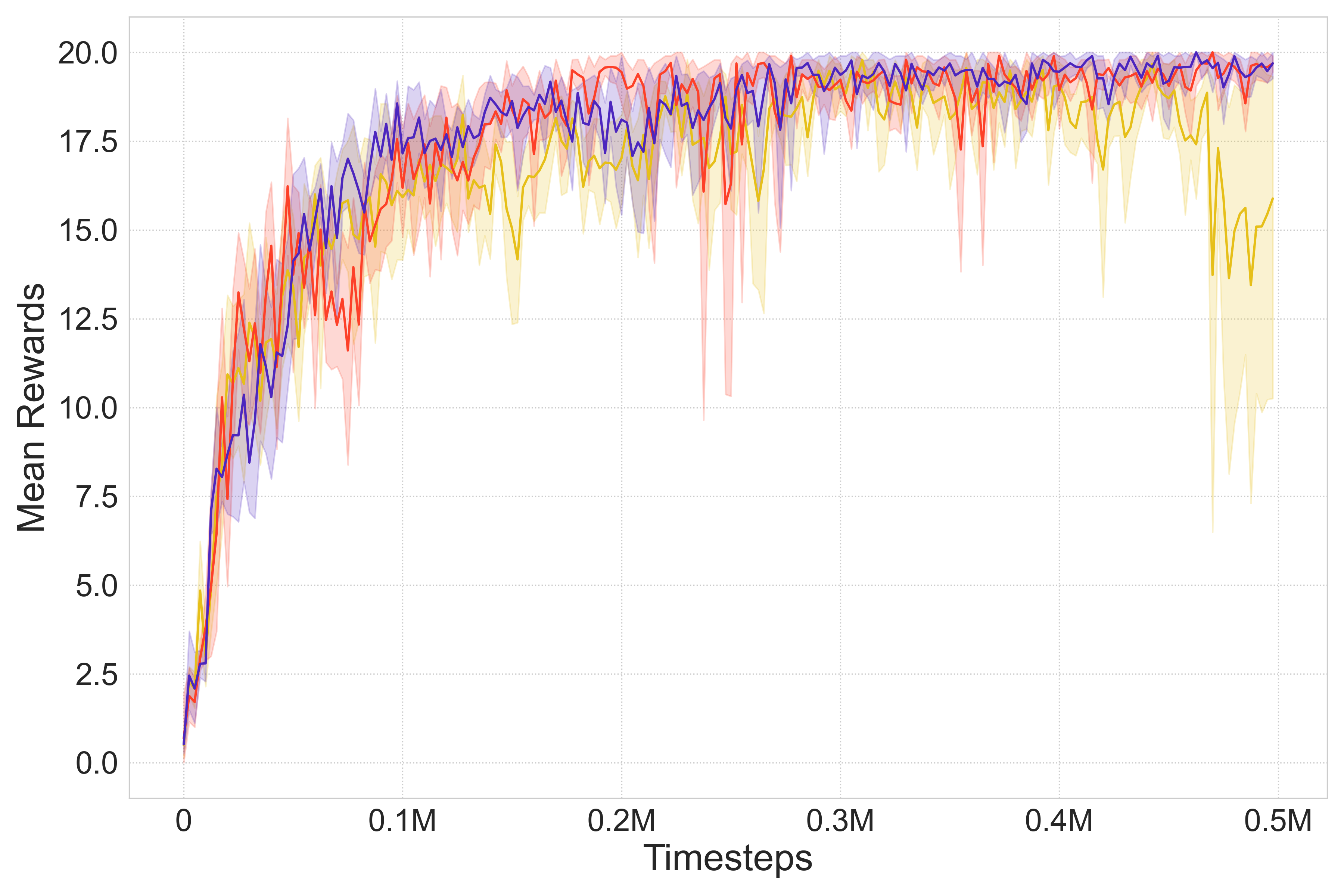}}
    \subfigure[SMAC-5m]{\label{fig:res_d}\includegraphics[width=0.23\textwidth]{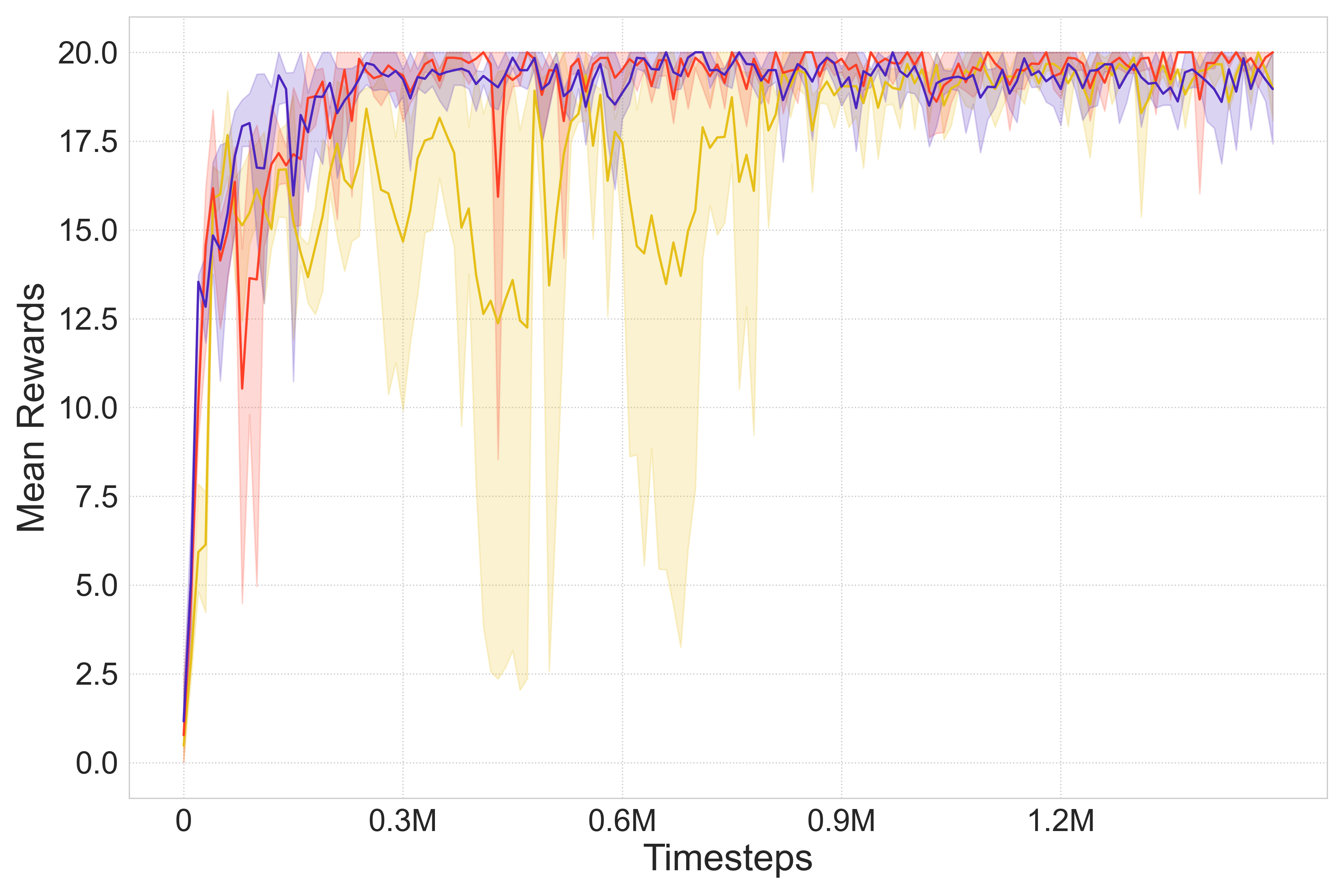}}
    \caption{Team rewards obtained for the experimented environments used to validate the ground truth method ICL described in section \ref{sec:icmarl}. The bold area represents the 95\% confidence intervals.}
    \label{fig:results}
\end{figure*}

\setcounter{footnote}{0} 
\section{Experiments and Results\protect\footnote{Codes used can be found at \href{https://github.com/rafaelmp2/causal-marl}{https://github.com/rafaelmp2/causal-marl}}}\label{sec:results}
In this section we evaluate the performance of the proposed approaches. Firstly, we intend to validate ICL by demonstrating that it can solve the described tasks, even with superior performance in most cases. IDQL, as described in section \ref{sec:idql}, and QMIX \cite{rashid_qmix_2018} are the two methods used to benchmark ICL. The motivations for the use of these two baselines are, respectively, to demonstrate how the causality detection used by the proposed method improves the quality of the behaviours learned for fully independent learners and to investigate whether using fully independent learning over CTDE methods like QMIX (that are centralised during training and share parameters) is a reasonable way to learn efficient behaviours, despite only using individual observations and independent networks. Additionally, we evaluate the individual behaviours learned by the agents and show that ACD-MARL can predict causality relations in MARL, using ICL relations as the ground truth.

\subsection{Rewards with ICL}\label{sec:results_1}
In Fig. \ref{fig:results} we can see the evolution of the team rewards obtained by the proposed method and the baselines QMIX and IDQL (average of 6 independent runs). Fig. \ref{fig:results}a and \ref{fig:results}c show that all the experimented methods learned policies that can solve Predator-Prey and SMAC-3m. Yet, when compared to QMIX we can see that ICL learns the tasks sooner, and remains then at the same level for the rest of the training time. On the other hand, IDQL stays slightly below the other methods in both environments. In the case of SMAC-5m, Fig. \ref{fig:results}d shows once again a faster initial boost of ICL that then remains at an optimal level. In this scenario, IDQL shows heavy variations of performance when compared to the others. These results suggest that the behaviours learned by ICL based on the individual intuition and causal perceptions of the agents led to more efficient and intelligent behaviours. We provide more details about the agent behaviours learned by ICL in appendix \ref{sec:app_indiv_icl}.

Fig. \ref{fig:results}b illustrates the performance of the attempted methods in Lumberjacks. This task can be easily solved by QMIX, demonstrating the usefulness of centralised training in this scenario while IDQL fails to achieve a good team reward in this task. Although ICL does not converge as quickly as QMIX in this environment, as the training progresses it reaches the same level as QMIX. Furthermore, when compared to IDQL, ICL demonstrates a much-increased performance. This enforces the fact that, when the agents are constrained to use just their local observations, using causal detection with ICL still allows to learn efficient policies, achieving the same level of performance of centralised training methods such as QMIX, despite the lack of full state information and parameter sharing.

Overall, ICL demonstrates a consistently good performance in all the environments. Exceptionally, ICL takes longer to converge in Lumberjacks, but eventually achieves the same level as QMIX. The presented results also lead to the conclusion that, under certain circumstances, centralised learning can be replaced with decentralised learning, together with some processing of the information gathered from the environment. Despite centralisation being useful in scenarios like Lumberjacks (Fig. \ref{fig:results}c), our results suggest that it is possible to learn in a fully decentralised manner, avoiding not only centralisation but also parameter sharing. While parameter sharing can be useful for training speed and resource matters, the possibility of avoiding this convention allows to create more suitable agents capable of learning independently to cooperate with others. For instance, when it comes to applications in real systems, the agents might not be able to use a common network that shares parameters or have a centralised oracle that shares common observations. In such scenarios, the proposed method can provide a useful independent learning mechanism based on causality.
% NOTE fix figure
\begin{figure}[!t]
    \centering
    \includegraphics[width=0.6\columnwidth]{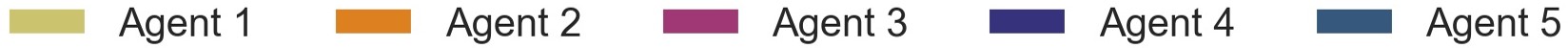}
    \\
    \vspace{0.00mm} 
    \subfigure[]{\label{fig:add_res_acd_a}\includegraphics[width=0.23\columnwidth]{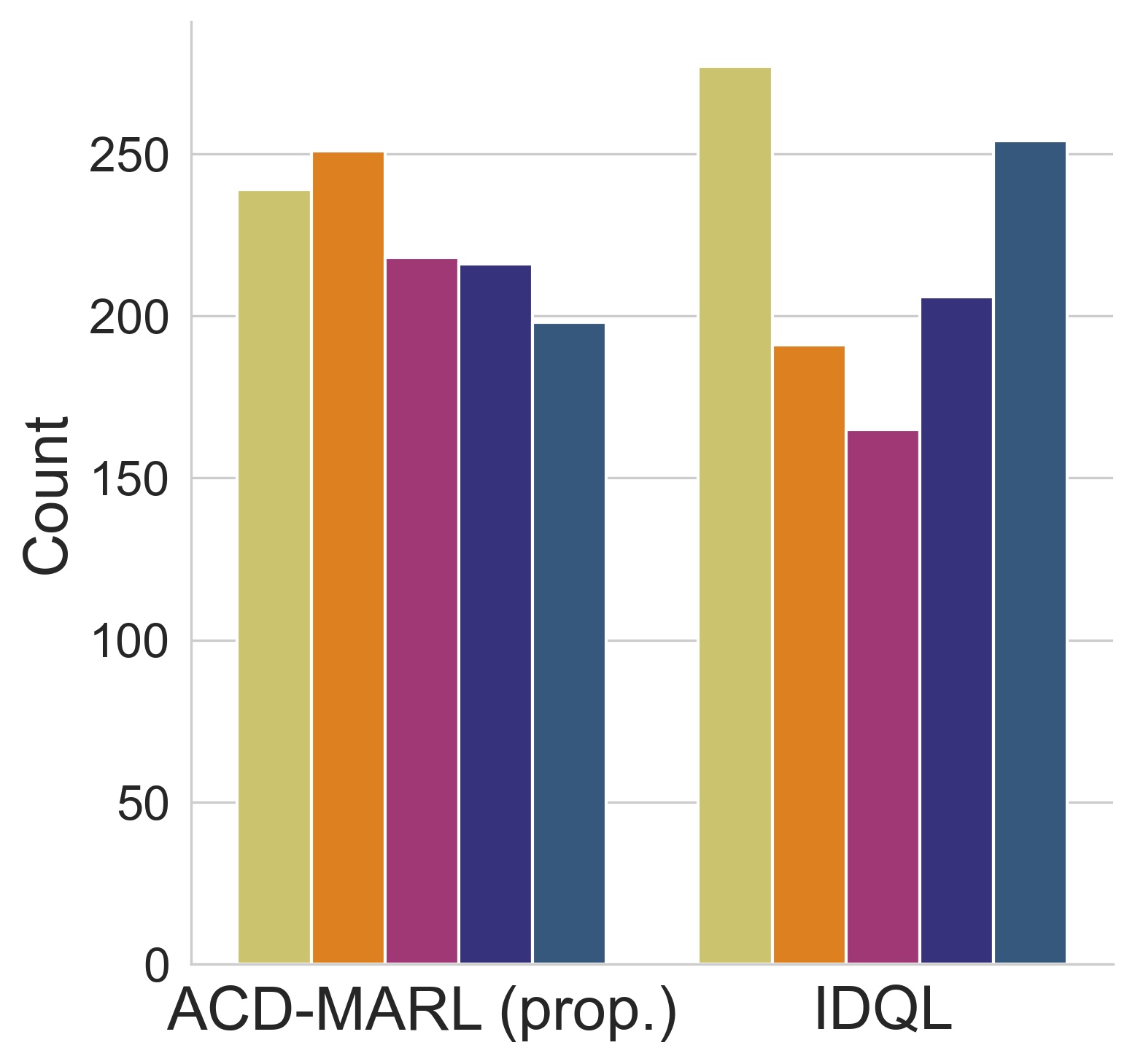}}
    \subfigure[]{\label{fig:add_res_acd_b}\includegraphics[width=0.23\columnwidth]{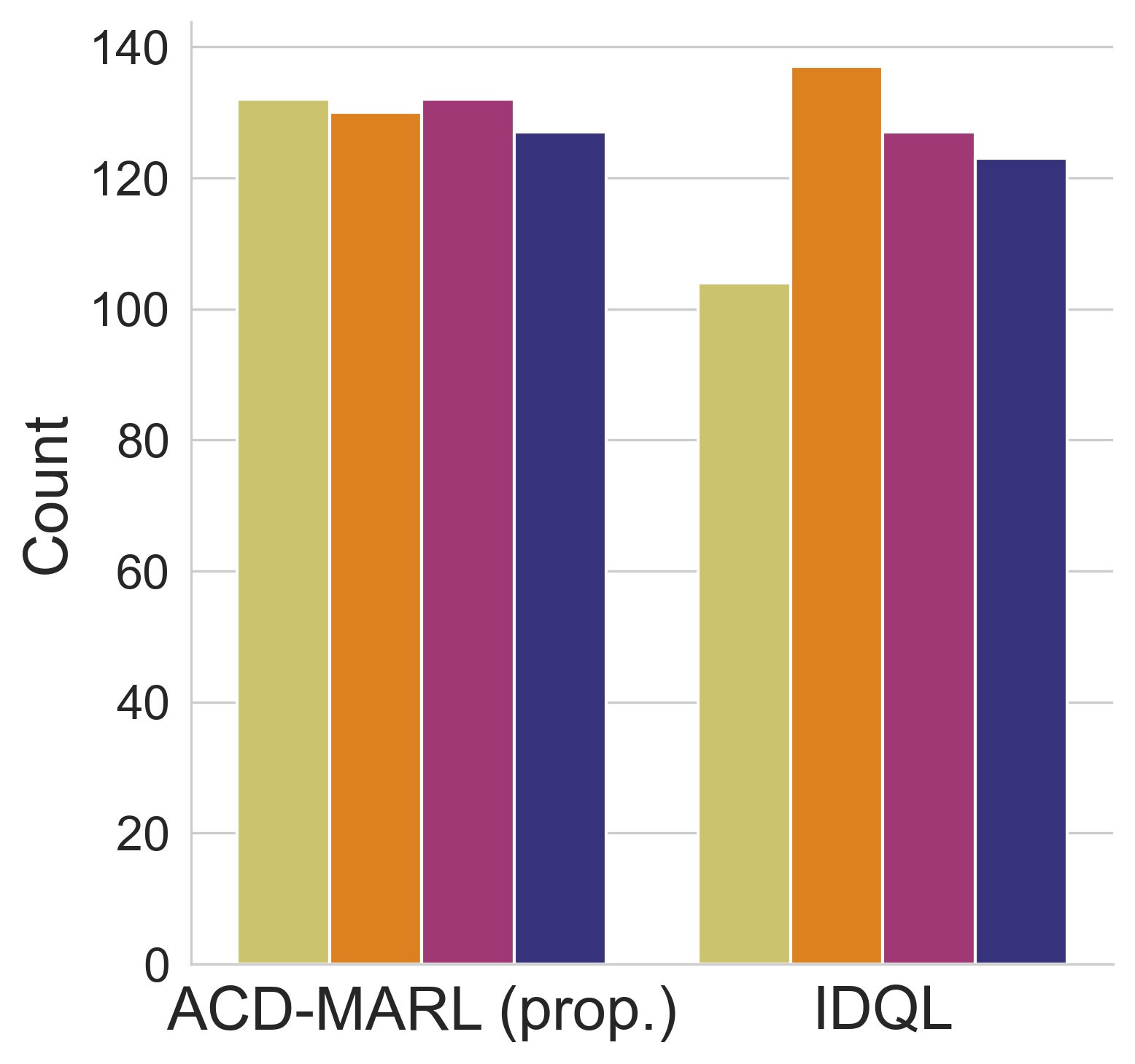}}
    \subfigure[]{\label{fig:add_res_acd_c}\includegraphics[width=0.23\columnwidth]{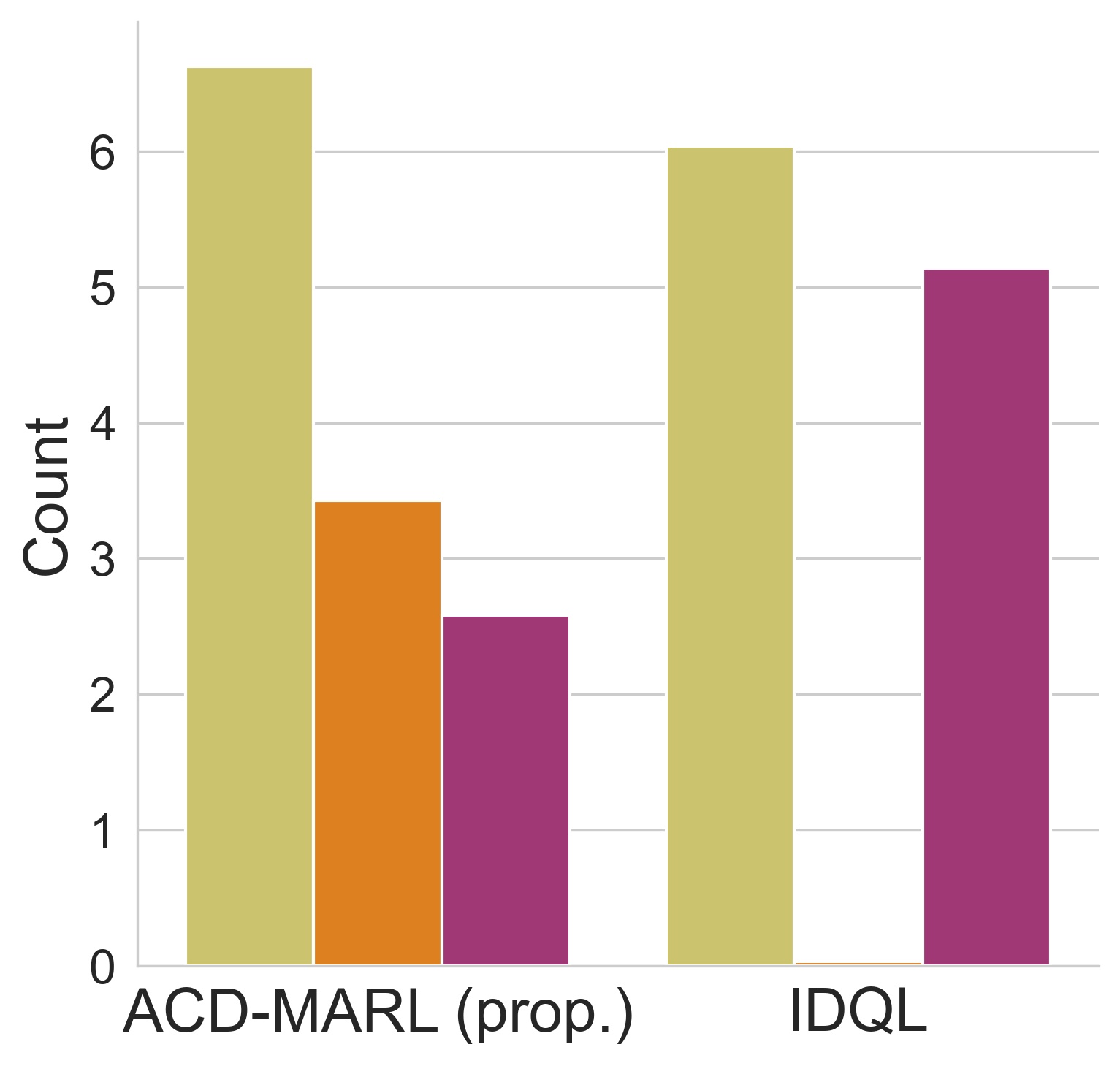}}
    \subfigure[]{\label{fig:add_res_acd_d}\includegraphics[width=0.23\columnwidth]{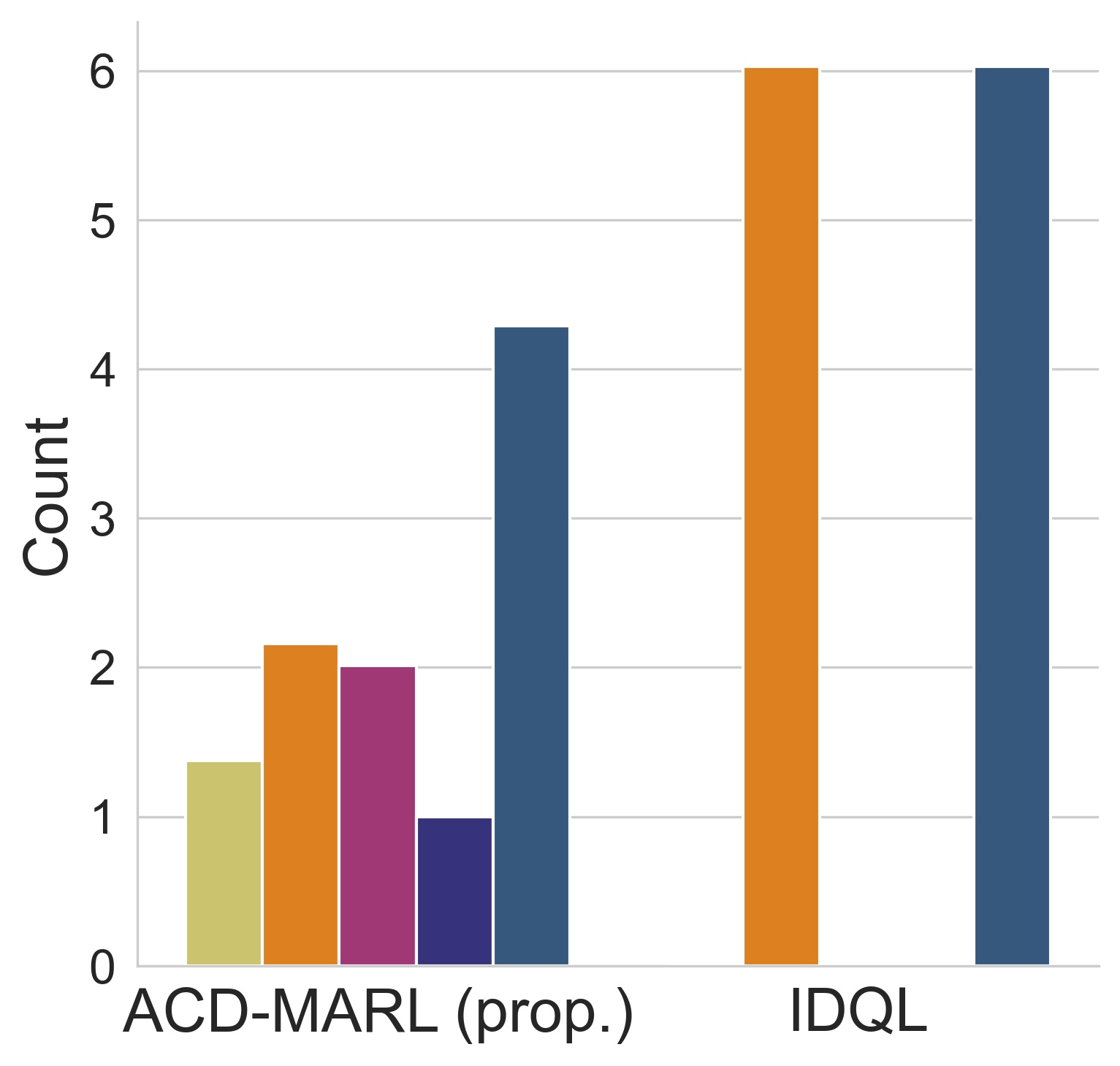}}
    \caption{Behaviour metrics with trained agents for ACD-MARL vs IDQL. (a) cumulative preys caught per agent in Predator-Prey-sp (300 episodes with 5 agents for result significance); (b) cumulative agent contributions to cut a tree in Lumberjacks-sp (300 episodes); (c) 300-episode mean shots fired per agent in SMAC-3m-sp and (d) in SMAC-5m-sp (in this case, we can see that for IDQL only 2 agents out of 5 perform shots).}
    \label{fig:add_results_acd}
\end{figure}

\subsection{Causality in ACD-MARL and Learned Behaviours}\label{sec:res_acd}
In the previous section we have demonstrated that detecting causal relations in MARL can lead to performance improvements. When done accurately enough, the agents can understand whether they were directly responsible or not for the rewards returned from the environment, improving the credit assignment.

In this section, we show the results of ACD-MARL aiming to support that causal relations are present in the dynamics of MARL problems. To evaluate the method we use the causality detection mechanisms of ICL presented in the previous sections \ref{sec:caus_envs} and \ref{sec:results_1} as the ground truth.

In order to evaluate the performance of ACD-MARL we use a set of sparser environments from the family of the ones represented in Fig. \ref{fig:maps}. In this set of environments, we modify the tasks in order to contain less positive reward signals. As a result, the individual performances of the agents will be much more affected, enhancing the impact of the proposed method in this paper. We opt for this configuration in order to investigate scenarios that allow a stronger analysis of the impact of causal discovery in MARL scenarios. In this section we refer to these environments as the name of the environment followed by the suffix "-sp". We provide further details and discussions in appendix \ref{sec:app_sparser_envs}. 

\begin{table}[h]
  \centering
  \caption{Accuracies of the $c_i$ predictions made by the modified ACD (Fig. \ref{fig:acd_arch}) compared against the validated methods of ICL described in section \ref{sec:caus_envs}. F.P. (False Positives) correspond to wrongly predict a $o\rightarrow r$ \textit{edge} in the graph and F.N. (False Negatives) to wrongly predict \textit{no edge} in the graph.}
  \label{tab:acd_comp}
  \begin{tabular}{cccc}\toprule
    Samples from & Correct (\%) & F.P. (\%) & F.N. (\%) \\ \midrule
    Predator-Prey & 68\% & 29\% & 3\% \\
    Lumberjacks & 67\% & 28\% & 5\% \\
    SMAC-3m & 95\% & 4\% & 1\% \\ 
    SMAC-5m & 87\% & 12\% & 1\% \\ \bottomrule
    \end{tabular}
\end{table}

The results in Table \ref{tab:acd_comp} show that the predicted $o\rightarrow r$ causal relations by ACD-MARL are close to the ground truth relations of ICL, enforcing that causal relations are indeed present in the underlying dynamics of MARL. The cases where ACD-MARL fails to predict correspond mostly to false positive cases. Following the lines of the method proposed in this paper to punish lazy agents using a factor $c_i$ (as per Theorem \ref{theo:theo1}), we note that these false positive cases are not as harmful for the performance of a MARL team as false negatives (that are very low) when we use the predictions to adjust the credit assignment. This is due to false positives being equivalent to simple independent learning. By having these sparsely, together with a high number of correct credit assignment causal predictions, it would still result in an improvement of team performance and individual behaviours. These observations are supported by the results in Fig. \ref{fig:add_results_acd}, showing that when we use the trained ACD to predict $c_i$ and update the agents as per Theorem \ref{theo:theo1}, the agents learn more intelligent behaviours. We can see that the individual behaviours of the ACD-MARL agents are more balanced both for Predator-Prey and Lumberjacks (all agents show equally high performances) opposing to IDQL, where we can see some discrepancies. In addition, in SMAC-3m and SMAC-5m, we can see that some of the agents of IDQL fail to learn cooperative policies and do not help the team, while ACD-MARL incentivises all the agents to contribute for the team goal, creating more intelligent behaviours. For instance, in SMAC-5m (Fig. \ref{fig:add_results_acd}d) only two agents learned to shoot enemies in IDQL, meaning that the others get lazy and do not learn cooperative policies (they simply roam around). On the contrary, ACD-MARL proves to be capable of training cooperative policies for all agents, where we can see that all of them help towards the team goal.

\section{Conclusion and Further Work}
This paper investigated the use of causality in MARL and how causal detection can be leveraged to improve learning. By learning an agent-wise causality relation across individual observations and the team reward, the credit assignment can be adjusted, resulting in the elimination of lazy agents that deteriorate the performance of the team due to their non-cooperative behaviours.

By formalising a MARL problem as a causality detection problem, we observe that ACD can be used to accurately model the causal dynamics across observations and rewards in MARL episodes. The results show accuracies close to the ground truth causal detections estimated by the first proposed method in this paper (ICL), enforcing the fact that causal relations are indeed present in MARL systems and can be inferred.

Overall, the results in this paper enhance the importance of investigating causality in the field of MARL, supporting discussions and hypotheses in previous works such as \cite{causal-marl-grimbly}. In the future, we intend to study how causality can be applied in MARL real scenarios. We believe that causality applications in MARL can be a breakthrough not only in simulation, but also in real scenarios where reliable communication or a centralised oracle is not available. In such cases, agents might have to coordinate from scratch and develop a sense of self causal-relation perception.

\bibliographystyle{unsrt}
%\bibliography{references}  %%% Uncomment this line and comment out the ``thebibliography'' section below to use the external .bib file (using bibtex) .

%%% Uncomment this section and comment out the \bibliography{references} line above to use inline references.
% \begin{thebibliography}{1}

% 	\bibitem{kour2014real}
% 	George Kour and Raid Saabne.
% 	\newblock Real-time segmentation of on-line handwritten arabic script.
% 	\newblock In {\em Frontiers in Handwriting Recognition (ICFHR), 2014 14th
% 			International Conference on}, pages 417--422. IEEE, 2014.

% 	\bibitem{kour2014fast}
% 	George Kour and Raid Saabne.
% 	\newblock Fast classification of handwritten on-line arabic characters.
% 	\newblock In {\em Soft Computing and Pattern Recognition (SoCPaR), 2014 6th
% 			International Conference of}, pages 312--318. IEEE, 2014.

% 	\bibitem{hadash2018estimate}
% 	Guy Hadash, Einat Kermany, Boaz Carmeli, Ofer Lavi, George Kour, and Alon
% 	Jacovi.
% 	\newblock Estimate and replace: A novel approach to integrating deep neural
% 	networks with existing applications.
% 	\newblock {\em arXiv preprint arXiv:1804.09028}, 2018.

% \end{thebibliography}

%% APPENDIX

\newpage
\appendix
\onecolumn
\makeatletter
\@addtoreset{theorem}{section}
\@addtoreset{assumption}{section}
\section{Proof of Theorem 1}\label{sec:app_proof}
Firstly, let us state Assumption \ref{ass:app_assump_1} and Theorem \ref{theo:app_theo1}:
\begin{assumption}\label{ass:app_assump_1}
\sloppy \textit{Given a set $E=\{o_1,\dots,o_N,r\}$ and for each agent-wise causality factor $c_i$, $c_i(o_i,r) = 1 \implies \exists o_i^{'\leq t}\neq o_i^{\leq t} : g_r(o_1^{\leq t},\dots, o_i^{'\leq t},\dots, o_N^{\leq t},r) \neq g_r(o_1^{\leq t},\dots, o_i^{\leq t},\dots, o_N^{\leq t},r)$ and $c_i(o_i,r) = 0 \implies \forall o_i'^{\leq t}\neq o_i^{\leq t} : g_r(o_1^{\leq t},\dots, o_i^{'\leq t},\dots, o_N^{\leq t},r)=g_r(o_1^{\leq t},\dots, o_i^{\leq t},\dots, o_N^{\leq t},r)$ (with terms as Definition 2).}
\end{assumption}

\begin{theorem}\label{theo:app_theo1}
For a certain MARL task with $N$ agents, if at a given timestep $t$ each agent $i:i \in \{1,\ldots,N\}$ calculates an individual binary causality factor $c_i$, under Assumption \ref{ass:app_assump_1}, where
\begin{equation}\label{eq:app_theo_eq1}
c_i(o_i,r)=\left\{
\begin{array}{ll}
    1 & o_i\ causes\ r\\
    0 & \lnot\ o_i\ causes\ r
\end{array}, i \in \{1,\ldots,N\}
\right.    
\end{equation}
and $o_i$ and $r$ denote the individual observations and the team reward at that timestep for an episode $E$, respectively, and each individual $Q_i$ is updated following the rule
\begin{equation}\label{eq:app_theo_eq2}
Q_i(\tau_i,a_i)=(1-\alpha)Q_i(\tau_i,a_i)+\alpha\left[c_i(\tau_i,r)\times r+\gamma\mathop{\mathrm{max}}_{a_i'}Q_i(\tau_i',a_i')\right]
\end{equation}
where $\alpha$ is the learning rate, then the convergence of the learning Q-function is preserved.
\end{theorem}

\begin{proof}
Consider a multi-agent cooperative scenario with $N$ agents. The Q-function for each agent $i:i\in \{1,\ldots,N\}$ is updated following the rule
\begin{align*}
    Q_i(\tau_i,a_i)=(1-\alpha)Q_i(\tau_i,a_i)+\alpha\left[r+\gamma\mathop{\mathrm{max}}_{a_i'}Q_i(\tau_i',a_i')\right]
\end{align*}
Let $Q_i^c$ represent the Q-function for an agent $i$ after applying its individual causality factor $c_i$ on the reward and $c_i$ is calculated as in Eq. (\ref{eq:app_theo_eq1}) and under Assumption \ref{ass:app_assump_1}. As from Eq. (\ref{eq:app_theo_eq2}), the update can be expanded as follows
\begin{align*}
Q_i^c(\tau_i,a_i)&=\left\{
\begin{array}{ll}
    (1-\alpha)Q_i^c(\tau_i,a_i)+\alpha\left[r+\gamma\mathop{\mathrm{max}}_{a_i'}Q_i^c(\tau_i',a_i')\right] & \tau_i\ causes\ r\\
    \\
    (1-\alpha)Q_i^c(\tau_i,a_i)+\alpha\gamma\mathop{\mathrm{max}}_{a_i'}Q_i^c(\tau_i',a_i') & \lnot\ \tau_i\ causes\ r
\end{array}
\right.
\\
\\
&=\left\{
\begin{array}{ll}
    Q_i^c(\tau_i,a_i;r\neq 0) & \tau_i\ causes\ r\\
    \\
    Q_i^c(\tau_i,a_i;r=0) & \lnot\ \tau_i\ causes\ r
\end{array}
\right.
\end{align*}
It is true that $\lim_{Ts \to \infty}c_i(o_i,r)\in\{0,1\}$ because $c_i(o_i,r)\in\{0,1\}$, where $Ts$ is a number of timesteps, and thus, in the same way that an unchanged Q-function $Q_i \equiv Q_i^c$ when $c_i$ is always 1 (from the above expansion) will converge to a local optimum, also $Q_i^c$ with a valid $c_i$ calculation under Assumption \ref{ass:app_assump_1} will converge to a local optimum $Q_i^{c*}$, i.e., $\lim_{Ts\to\infty}(Q_i^{c*}-Q_i^c)=0$. 
\end{proof}

\section{Individual Behaviours that are Learned with ICL} \label{sec:app_indiv_icl}
In this paper we have introduced ICL as the ground truth used to benchmark the causality predictions made by ACD-MARL. To validate the method, we have demonstrated in section \ref{sec:icmarl} how ICL can solve a set of different environments with different levels of complexity (Fig. \ref{fig:maps}). To support the better performances achieved as a team by this method when compared to simple independent learners, we present here a deeper discussion similar to the one made to analyse the improvements in individual performances when using ACD-MARL. To this end, we have selected trained policies of IDQL and ICL and evaluated the performances of the agents in the environments. This experiment aims to demonstrate how ICL enables independent agents to learn more intelligent behaviours when compared to simple independent learners.

Fig. \ref{fig:add_results}a depicts the number of shots fired by each agent in SMAC-3m and Fig. \ref{fig:add_results}b in SMAC-5m. The reward of these tasks is built based on the damage dealt to opponents and on the winning condition of the game. Thus, to maximise the reward it is important that all the agents participate and deal damage. ICL is capable of training this behaviour, as it can be seen in Fig. \ref{fig:add_results}a and \ref{fig:add_results}b. Although in IDQL some of the agents shoot more often that in ICL, the shooting distributions in ICL are more balanced for both scenarios, suggesting that more agents are cooperating in the task due to a better credit assignment, leading to a faster convergence and stable reward as shown previously in Fig. \ref{fig:results}c and \ref{fig:results}d of the results section. This shows that ICL encourages the individual elements of the team to cooperate more, punishing lazy agents. 

Fig. \ref{fig:add_results}c illustrates the cumulative proximity (euclidean distance) of each agent to the teammates during a successful trained episode of Lumberjacks for ICL and IDQL. Since in this task multiple agents are usually needed to chop a tree down, ICL encourages the agents to learn policies that will make them move always close to each other, making the task easier. As it can be seen in Fig. \ref{fig:add_results}c, in ICL the cumulative distances are much smaller for every agent, meaning that they try to stick close to the teammates while solving the task. On the other hand, for IDQL we can see that the cumulative distance for each agent is much higher, meaning that they are mostly far away from each other during the episode. As a result, this leads to worse performances as a team, failing to achieve optimal rewards as it was discussed in section \ref{sec:results_1}.

For the Predator-Prey task, Fig. \ref{fig:add_results}d shows the existence of non-cooperative agents in IDQL that are eliminated in ICL. With the proposed method it is possible to see not only a higher number of preys being captured but also a more balanced distribution of this number across the agents. This result demonstrates that ICL encourages the agents to participate more in the team objective, leading to a more consistently high reward and faster convergence, as shown previously in the results section, Fig. \ref{fig:results}.

\begin{figure}[!t]
    \centering
    \includegraphics[width=0.5\textwidth]{resources/legend_acd_exps.jpg}
    \\
    \vspace{0.00mm} 
    \subfigure[]{\label{fig:add_res_a}\includegraphics[width=0.23\textwidth]{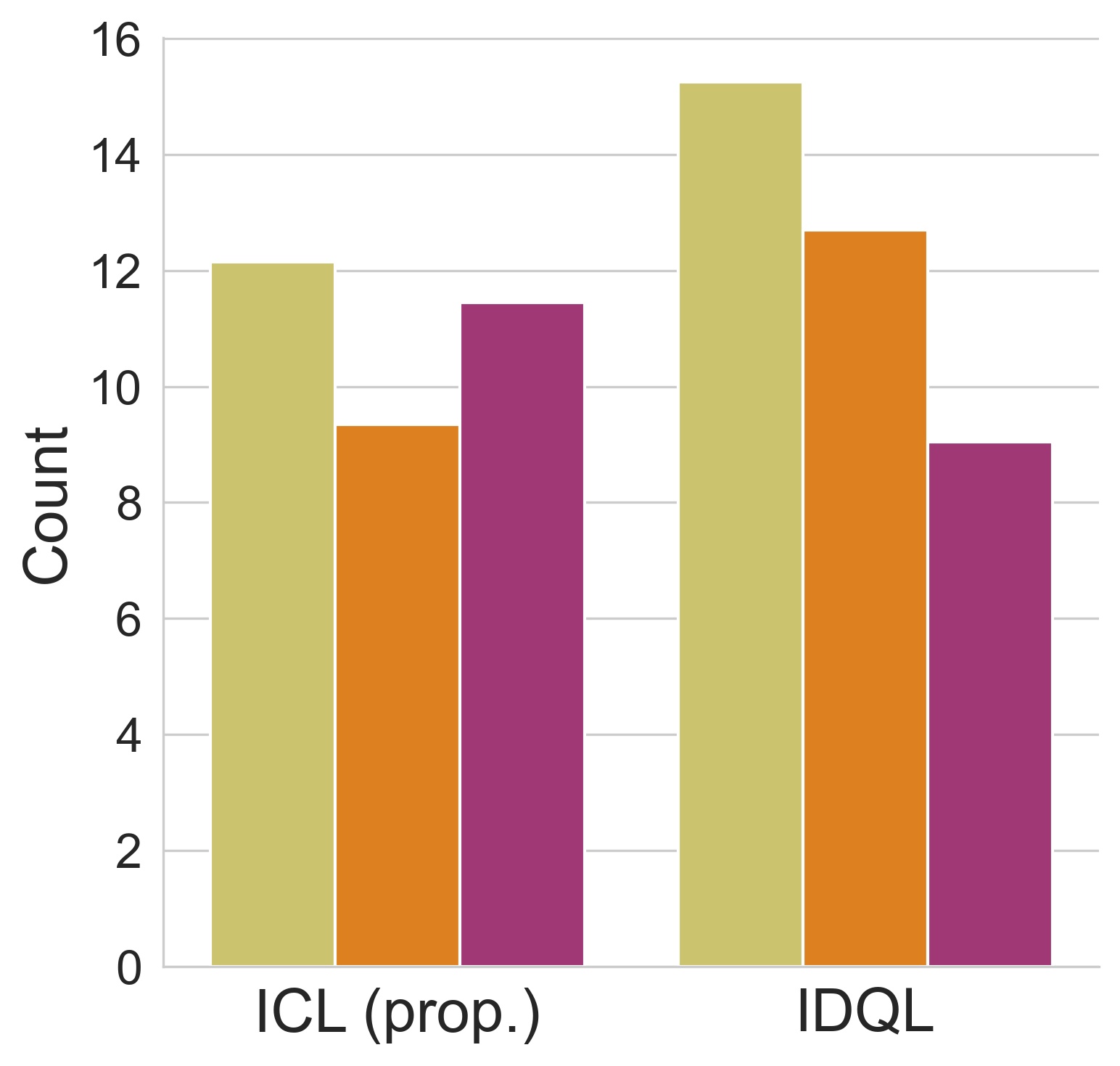}}
    \subfigure[]{\label{fig:add_res_b}\includegraphics[width=0.23\textwidth]{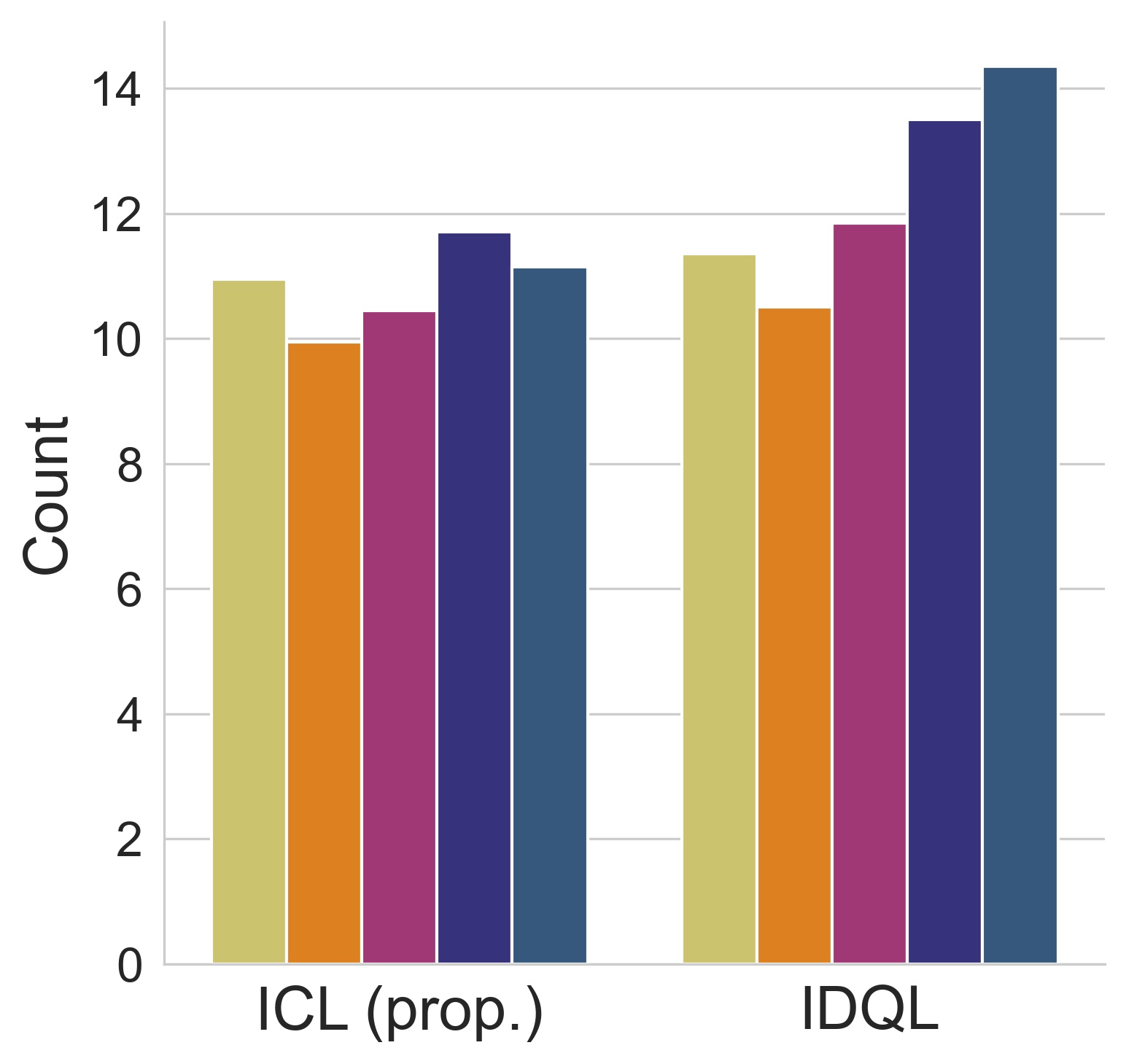}}
    \subfigure[]{\label{fig:add_res_c}\includegraphics[width=0.23\textwidth]{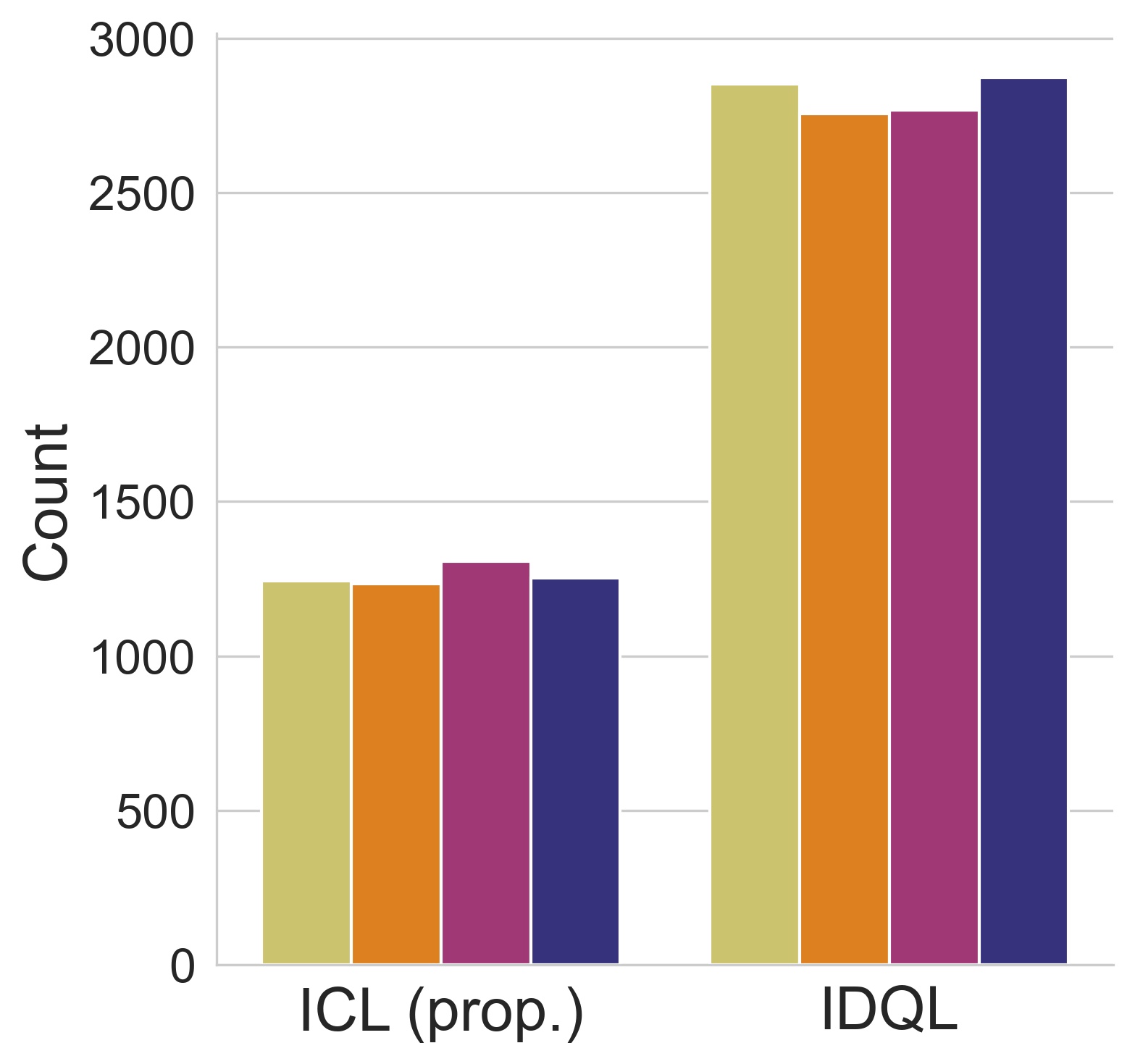}}
    \subfigure[]{\label{fig:add_res_d}\includegraphics[width=0.23\textwidth]{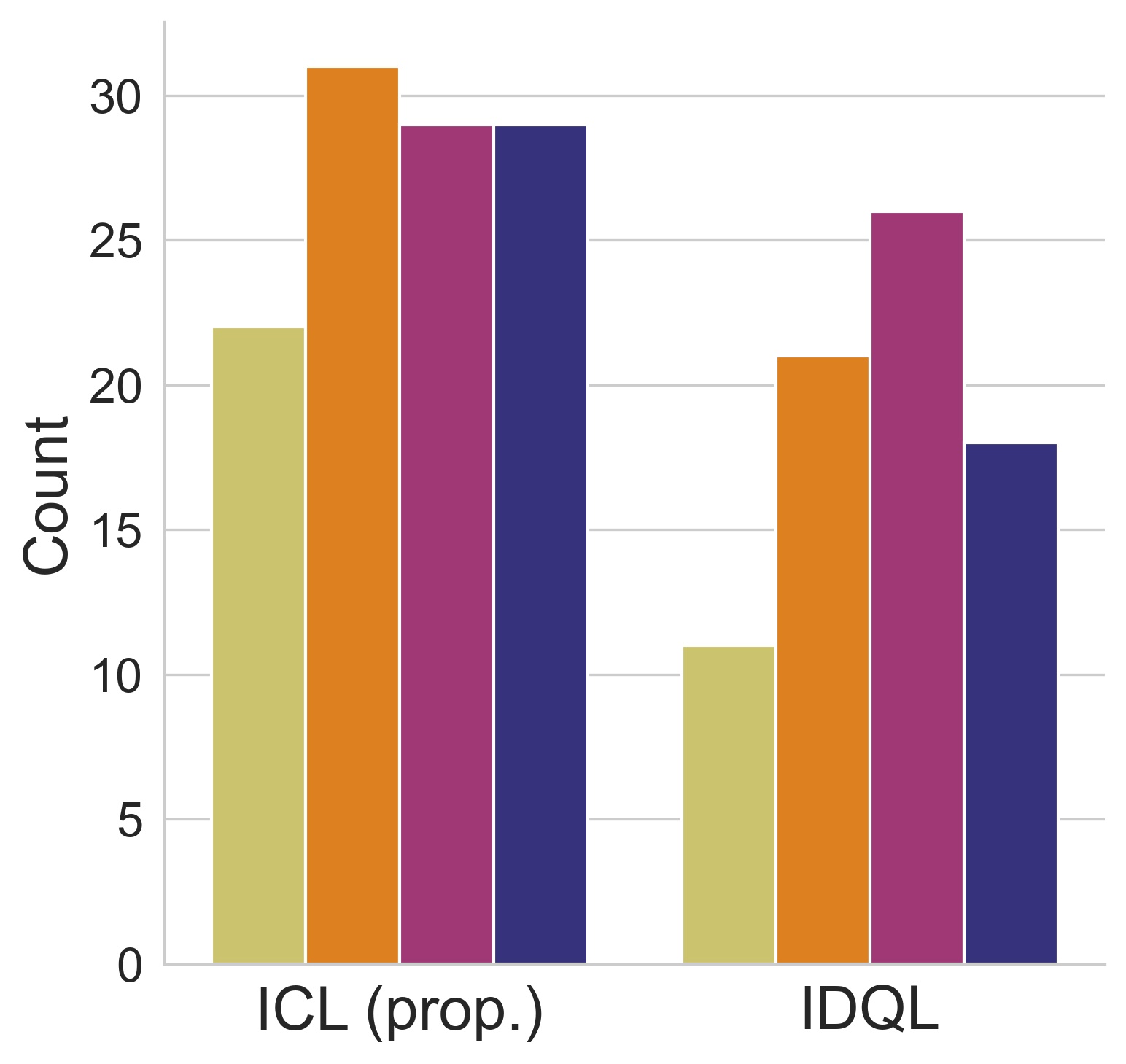}}
    \caption{Behaviour metrics with trained agents for ICL vs IDQL in the experimented environments. (a) 20-episode mean shots fired per agent in SMAC-3m and (b) in SMAC-5m; (c) cumulative distance from each agent to the other agents over one trained episode of Lumberjacks; (d) cumulative preys caught per agent in Predator-Prey (50 episodes).}
    \label{fig:add_results}
\end{figure}

\section{Further Discussions on the used Modified Environments}\label{sec:app_sparser_envs}
As described in section \ref{sec:res_acd}, in order to investigate in detail the impact of causal discovery in MARL with ACD-MARL, we modify the environments described in section \ref{sec:icmarl} to make them contain less positive reward signals. In this sense, Predator-Prey-sp corresponds to the Predator-Prey environment in Fig. \ref{fig:maps}a but with 5 agents and only 1 prey; Lumberjacks-sp corresponds to the Lumberjacks environment in Fig. \ref{fig:maps}b with the same 4 agents but only one tree in the map; SMAC-3m-sp and SMAC-5m-sp correspond to the same environments in Fig. \ref{fig:maps}c and Fig. \ref{fig:maps}d (3 and 5 agents, respectively) but the agents play against only one enemy. While the global goal of the environments in this case becomes easier for the team, the resulting behaviours of the individual agents will be much more affected.

As demonstrated in this paper, it is important to note that being able to solve a task as a team does not necessarily translate to the emergence of good cooperative individual behaviours by every agent in the team. For instance, in the experiments in this paper, we have demonstrated that Independent Deep Q-learning (IDQL) can solve most of the environments but the individual behaviours learned are not as good from a cooperative point of view (as discussed in section \ref{sec:results} and complemented in the appendix \ref{sec:app_indiv_icl} in Fig. \ref{fig:add_results}). In the environments described previously in section \ref{sec:meth}, the nature of the tasks imposed requires a larger number of agents to solve the environments. However, in the modified environments, part of the agents might not be needed for the team to accomplish the goal. Logically, this will have a strong negative impact in the emergent behaviours of the agents, making them more prone to develop lazy behaviours. This new challenge opens possibilities of using types of transfer learning to boost the learning process. For instance, if we would train a set of agents in an environment that they can solve quickly, but not all of them are needed (for instance, in the case of SMAC-5m-sp), then they could learn it in a short amount of time and use the learned knowledge to improve learning in a more complex related scenario where all of them must help. However, if the agents develop lazy behaviours and do not learn cooperative policies, they will not be able to help the team when placed together with other agents in a different scenario. On the other hand, if the agents can learn cooperative behaviours, even if they are not necessarily needed, they will still learn how they can cooperate with others and help the team. 

As a simple additional example, we have trained QMIX in SMAC-5m-sp. Interestingly, QMIX fails to solve this environment. In Fig. \ref{fig:5m_vs_1m_qmix} we can see that it shows a very unstable performance, meaning that the agents get very confused in this type of environment. As a result, it fails to achieve the optimal performance that ACD-MARL can achieve (optimum in the figure). This suggests that this type of environments creates a different kind challenge that will make it very confusing for the agents to learn optimal behaviours, failing to cooperate and win the game.
\begin{figure}[H]
    \centering
    \includegraphics[width=0.3\textwidth]{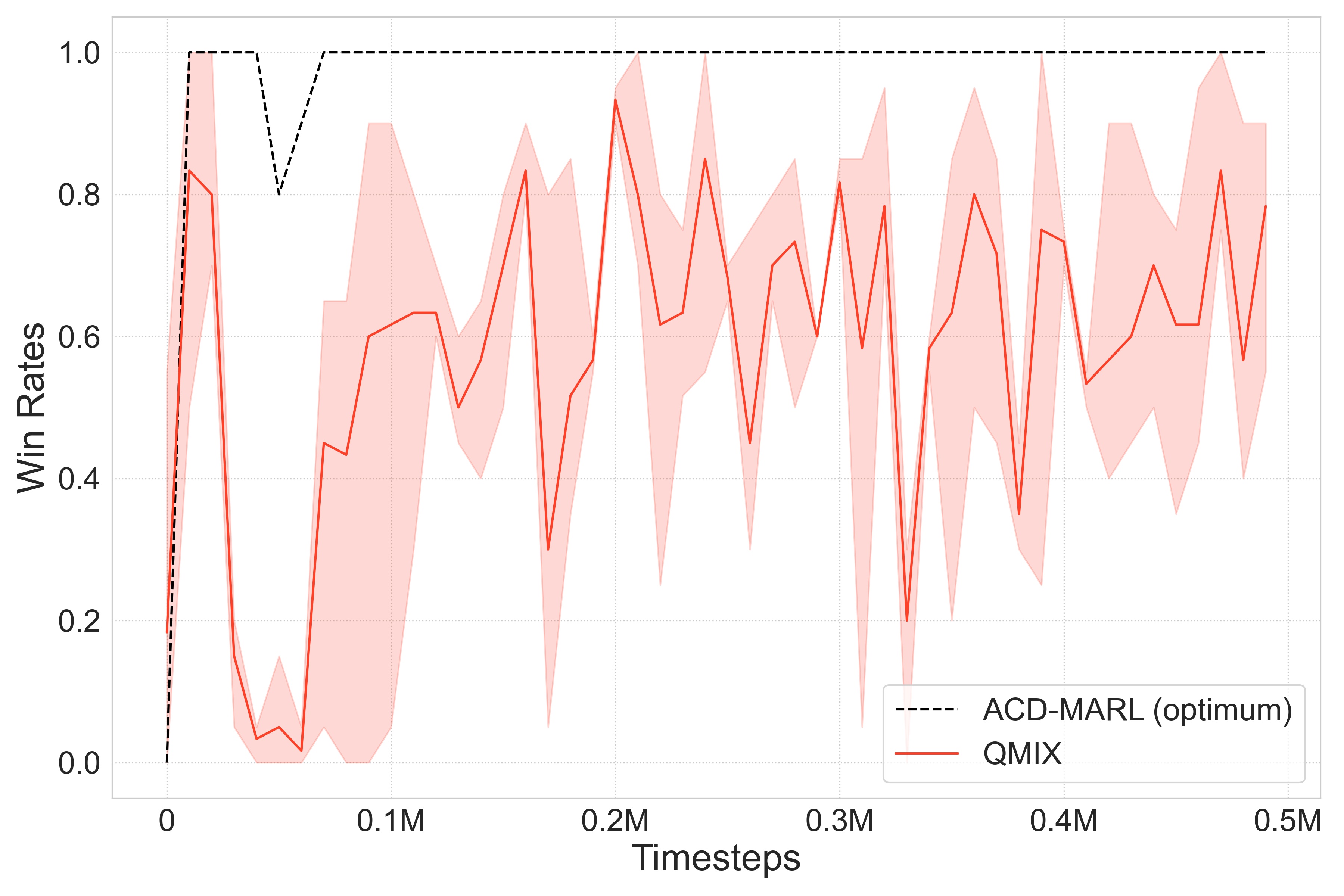}
    \caption{Win rate achieved by QMIX (average of 3 independent runs) in SMAC-5m-sp compared to the optimal win rate achieved by ACD-MARL. QMIX fails to learn this environment consistently.}
    \label{fig:5m_vs_1m_qmix}
\end{figure}

\section{Corresponding Win Rates for the SMAC Environments Presented in the Results Section}\label{sec:app_win_rates}
In section \ref{sec:results} we presented the obtained rewards for a diverse set of environments. Since we use more than only SMAC environments, we have opted to include only the rewards to keep consistency of the presentation of results. Nevertheless, for completeness we include in this appendix the corresponding win rates for the SMAC environments SMAC-3m and SMAC-5m.
\begin{figure}[H]
    \centering
    \includegraphics[width=0.3\textwidth]{resources/legend.jpg}
    \\
    \vspace{0.00mm} 
    \subfigure[SMAC-3m]{\label{fig:app_wr_1}\includegraphics[width=0.23\textwidth]{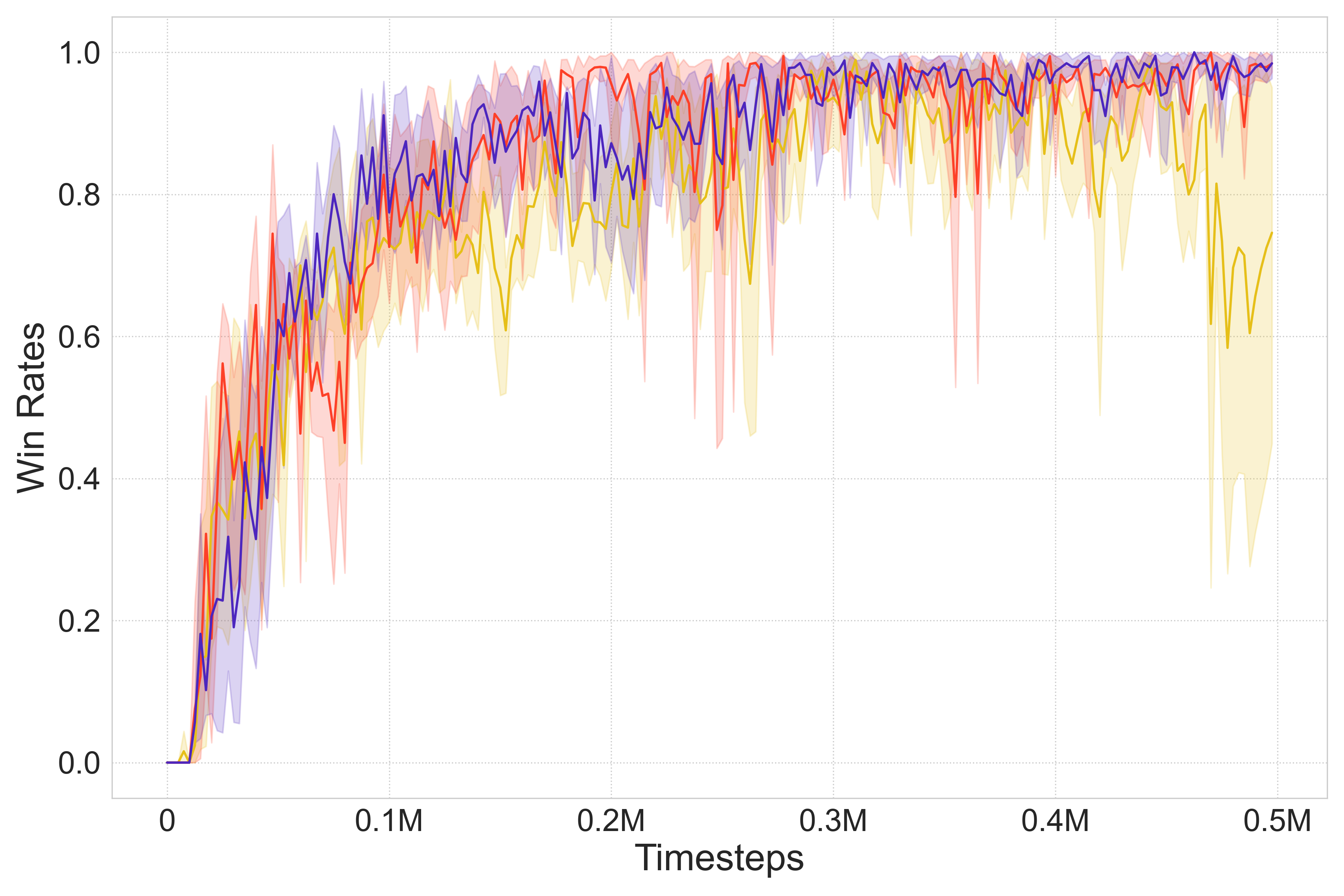}}
    \hspace{4mm}
    \subfigure[SMAC-5m]{\label{fig:app_wr_2}\includegraphics[width=0.23\textwidth]{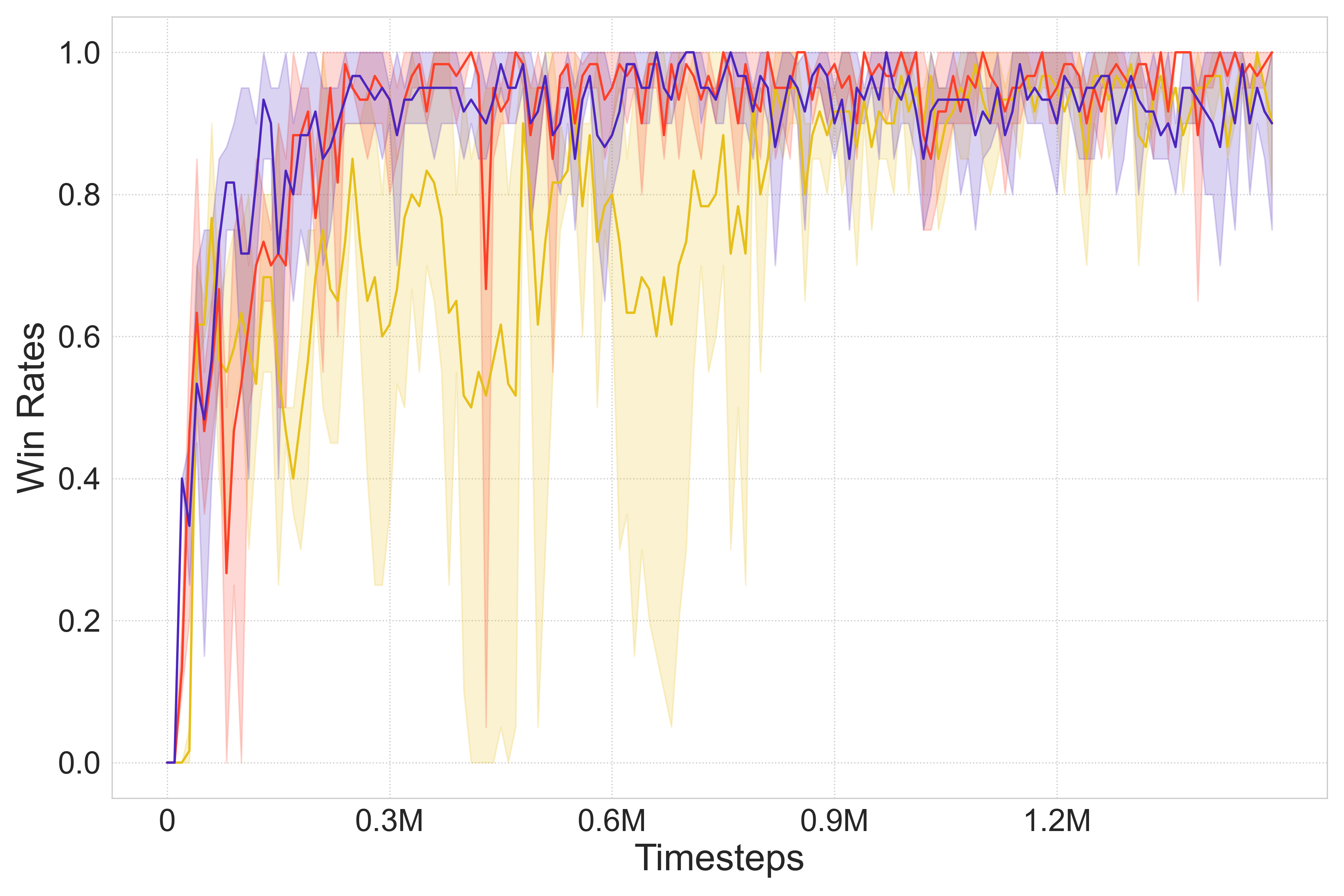}}
    \caption{Corresponding win rates achieved in SMAC-3m and SMAC-5m by the experimented methods as discussed in section \ref{sec:results}.}
    \label{fig:app_win_rates_fig}
\end{figure}

\section{Illustration of Parameter Sharing vs No Parameter Sharing Configurations}\label{sec:app_param_share_vs_no_param_share}
Since we have stated that we opt to use independent networks for the agents instead of using the parameter sharing convention (except for QMIX), we give an overview of the two options for clarity and completeness purposes. Figure \ref{fig:param_share} illustrates a high-level representation of the differences between parameter sharing and no parameter sharing architectures.  
\begin{figure}[H]
    \centering
    \includegraphics[width=\textwidth]{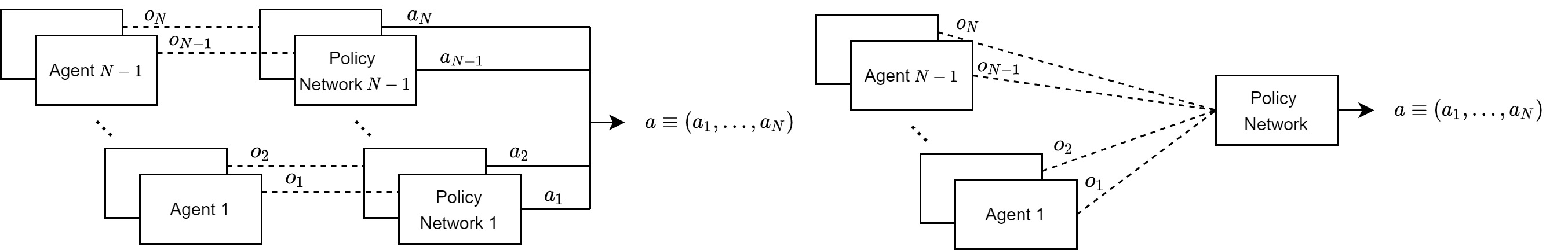}
    \caption{Overview of no parameter sharing (left) vs the use of parameter sharing, as methods like QMIX do, for example (right). Without parameter sharing agents use independent networks to output individual actions, given observations.}
    \label{fig:param_share}
\end{figure}

\section{Implementation Details and Hyperparameters Used in the Experiments}\label{sec:app_impl_details}
In the experiments with Multi-Agent Reinforcement Learning, the methods were trained for a different number of steps depending on the complexity of the environment (as shown in the Figures of the main paper), but all the methods are trained for 6 independent runs in each one of the environments. The evaluation cycle (interval at which the models are evaluated) also varies with the environment due to the varying number of training steps (2500 for SMAC-3m, and 10000 for Predator-Prey, Lumberjacks, and SMAC-5m).

In the implementation of the algorithms used, all the agents are controlled by recurrent deep neural networks that use a GRU (gated recurrent unit) cell with width 64. In the case of QMIX \cite{rashid_qmix_2018}, the agents always share the same policy network, but in IDQL, ICL and ACD-MARL each agent uses a separate independent policy network. Also in the implementation of QMIX, all the hidden layers used by the mixing network architecture have 32 units. 

The target networks are all updated every 200 training episodes. The optimiser used to train the networks of all the methods is the RMSprop optimiser, with learning rate ($\alpha$) set to $5\times 10^{-4}$. The discount factor ($\gamma$) used has value 0.99. The size of the replay buffer is set to 5000 and the minibatches sampled have a maximum size of 32 episodes. The exploration-exploitation tradeoff of the agents follows the epsilon-greedy method. The value of epsilon starts in 1 and anneals down to a minimum of 0.05 over 50000 training episodes.

In the experiments with Amortized Causal Discovery \cite{lowe_amortized_2022}, we start by saving winning episodes from the environments experimented. The samples saved are composed by sets of individual $N$ individual observations and the team reward. The episodes are processed as described in the main paper, and then given to the ACD framework. Note that each episodes contains $N$ series of observations (1 per agent) plus one series for the team reward, resulting in a total of $N+1$ series per episode. The episode length for Predator-Prey and Lumberjacks is 100, for SMAC-5m is 70, and for SMAC-3m is 60.

Regarding the hyperparameters of the ACD framework, we train the model for 150 epochs for each set of samples from each environment in separate and independently, and define a batch size of 128. For the lower bound loss used mentioned in the main paper, we define the prior of the KL divergence as a uniform categorical distribution to predict the causal relations. For the reconstructed term, we estimate the error using the equation \cite{kipf_neural_2018}
$$
-\sum_i\sum_{j}^t\frac{\left(y_i^j-y_i'^j\right)^2}{2\times\sigma ^2} + \text{const}
$$
where $y$ and $y'$ represent the predicted values and the targets, respectively, and $\sigma$ is a variance term. We have set the variance to $5\times10^{-4}$ for Predator-Prey and Lumberjacks and $5\times10^{-3}$ for SMAC-3m and SMAC-5m. Finally, we have used an RNN-based decoder to predict the next steps in our experiments.

To apply the predicted causal relations from ACD in MARL with ACD-MARL, we used the same MARL hyperparameters mentioned above in this supplementary section.

\end{document}